\documentclass{article}

\usepackage{microtype}
\usepackage{graphicx}
\usepackage{subfigure}
\usepackage{booktabs} % for professional tables

\usepackage{amsmath}
\usepackage{amsthm}
\usepackage{amssymb}
\usepackage{enumitem}   
\usepackage{mathrsfs}

\usepackage{hyperref}

\usepackage[accepted]{icml2018}

\newcommand{\R}{\mathbb{R}}
\newcommand{\Z}{\mathbb{Z}}

\DeclareMathOperator*{\conv}{conv}

\DeclareMathOperator*{\argmax}{arg\,max}

\DeclareMathOperator*{\sgn}{sign}
\newcommand{\bftheta}{\boldsymbol{\theta}}
\newcommand{\bfphi}{\boldsymbol{\phi}}
\newcommand{\bfx}{\boldsymbol{x}}
\newcommand{\bfp}{\boldsymbol{p}}

\newtheorem{theorem}{Theorem}

\newtheorem{lemma}{Lemma}
\newtheorem{proposition}{Proposition}

\newtheorem{remark}{Remark} 

\icmltitlerunning{An Optimal Control Approach to Deep Learning}

\begin{document}

\twocolumn[
\icmltitle{An Optimal Control Approach to Deep Learning and  \\
           Applications to Discrete-Weight Neural Networks}



\icmlsetsymbol{equal}{*}

\begin{icmlauthorlist}
\icmlauthor{Qianxiao Li}{ihpc}
\icmlauthor{Shuji Hao}{ihpc}
\end{icmlauthorlist}

\icmlaffiliation{ihpc}{Institute of High Performance Computing, Singapore}

\icmlcorrespondingauthor{Qianxiao Li}{liqix@ihpc.a-star.edu.sg}

\icmlkeywords{Deep learning, optimal control, Pontryagin's maximum principle, binary networks, ternary networks}

\vskip 0.3in
]


 
\printAffiliationsAndNotice{}  % leave blank if no need to mention equal contribution

\begin{abstract}
Deep learning is formulated as a discrete-time optimal control problem. 
This allows one to characterize necessary conditions for optimality and develop
training algorithms that do not rely on gradients
with respect to the trainable parameters. 
In particular, we introduce the discrete-time method of successive approximations (MSA),
which is based on the Pontryagin's maximum principle, for training neural networks. 
A rigorous error estimate for the discrete MSA is obtained, which sheds light on
its dynamics and the means to stabilize the algorithm. 
The developed methods are applied to train, in a rather principled way, neural networks with
weights that are constrained to take values in a discrete set.
We obtain competitive performance and interestingly, very sparse weights
in the case of ternary networks, which may be useful in model deployment in low-memory devices.
\end{abstract}

\section{Introduction}
\label{sec:intro}
The problem of training deep feed-forward neural networks is often studied as a nonlinear programming problem~\cite{bazaraa2013nonlinear,bertsekas1999nonlinear,kuhn2014nonlinear}
\[
  \min_{\bftheta} J(\bftheta)
\]
where $\bftheta$ represents the set of trainable parameters and $J$ is the empirical loss function. 
In the general unconstrained case, necessary optimality conditions are given by the condition $\nabla_{\bftheta} J(\bftheta^*)=0$ for an optimal set of training parameters $\bftheta^*$. This is largely the basis for (stochastic) gradient-descent based optimization algorithms in deep learning~\cite{robbins1951stochastic,duchi2011adaptive,zeiler2012adadelta,kingma2014adam}. When there are additional constraints, e.g.~on the trainable parameters, 
one can instead employ projected versions of the above algorithms. More broadly, 
necessary conditions for optimality can be derived in the form of the Karush-Kuhn-Tucker conditions~\cite{kuhn2014nonlinear}. 
Such approaches are quite general and typically do not rely on the structures of the objectives encountered in deep learning. However, in deep learning, the objective function $J$ often has a specific structure; It is derived from feeding a batch of inputs recursively through a sequence of trainable transformations, which can be adjusted so that the final outputs are close to some fixed target set. This process resembles an optimal control problem~\cite{bryson1975applied,bertsekas1995dynamic,athans2013optimal} that originates from the study of the calculus of variations. 

In this paper, we exploit this optimal control viewpoint of deep learning. First, we introduce the discrete-time Pontryagin's maximum principle (PMP)~\cite{halkin1966maximum}, which is an extension the central result in optimal control due to Pontryagin and coworkers~\cite{boltyanskii1960theory,pontryagin1987mathematical}. This is an alternative set of necessary conditions characterizing optimality, and we discuss the extent of its validity in the context of deep learning. Next, we introduce the discrete method of successive approximations (MSA) based on the PMP to optimize deep neural networks. A rigorous error estimate is proved that elucidates the dynamics of the MSA, and aids us in designing optimization algorithms under rather general conditions. We apply our method to train a class of unconventional networks, i.e.~those with discrete-valued weights, to illustrate the usefulness of this approach. In the process, we discover that in the case of ternary networks, our training algorithm obtains trained models that are very sparse, which is an attractive feature in practice. 

The rest of the paper is organized as follows: In Sec.~\ref{sec:opt_ctrl_viewpoint}, we introduce the optimal control viewpoint and the discrete-time Pontryagin's maximum principle. We then introduce the method of successive approximation in Sec.~\ref{sec:msa} and prove our main estimate, Theorem~\ref{thm:error_est}. In Sec.~\ref{sec:application}, we derive algorithms based on the developed theory to train binary and ternary neural networks. Finally, we end with a discussion on related work and a conclusion in Sec.~\ref{sec:discussion} and~\ref{sec:conclusion} respectively. Various details on proofs and algorithms are provided in Appendix A-D, which also contains a link to a software implementation of our algorithms that reproduces all experiments in this paper. 

Hereafter, we denote the usual Euclidean norm by $\Vert\cdot\Vert$ and the corresponding induced matrix norm by $\Vert\cdot\Vert_2$. The Frobenius norm is written as $\Vert\cdot\Vert_F$. Throughout this work, we use a bold-faced version of a variable to represent a collection of the same variable, but indexed additionally by $t$, e.g.~$\bftheta:=\{\theta_t:t=0,\dots,T-1\}$. 

\section{The Optimal Control Viewpoint}
\label{sec:opt_ctrl_viewpoint}
In this section, we formalize the problem of training a deep neural network as an optimal control problem. 
Let $T\in\Z_+$ denote the number of layers and $\{x_{s,0}\in\R^{d_0}: s=0,\dots,S\}$ represent a collection of fixed inputs (images, time-series). Here, $S\in \Z_+$ is the sample size. Consider the dynamical system
\begin{equation}
  x_{s,t+1} = f_t(x_{s,t},\theta_t), \quad t=0,1,\dots,T-1,
  \label{eq:dyn_sys}
\end{equation}
where for each $t$, $f_t: \R^{d_t} \times \Theta_t \rightarrow \R^{d_{t+1}}$ is a transformation on the state.
For example, in typical neural networks, it can represent a trainable affine transformation or a non-linear activation (in which case it is not trainable and $f_t$
does not depend on $\theta$). 
We assume that each trainable parameter set $\Theta_t$ is a subset of an Euclidean space. 
The goal of training a neural network is to adjust the weights $\bftheta:=\{\theta_t:t=0,\dots,T-1\}$ so as to minimize some loss function that measures the difference between the final network output $x_{s,T}$ and some true targets $y_s$ of $x_{s,0}$, which are fixed. Thus, we may define a family of real-valued functions $\Phi_s:\R^{d_T} \rightarrow\R$ acting on $x_{s,T}$ ($y_s$ are absorbed into the definition of $\Phi_s$) and the average loss function is $\sum_s \Phi_s(x_{s,T})/S$. In addition, we may consider some regularization terms for each layer $L_t:\R^{d_t}\times\Theta_t\rightarrow\R$ that has to be simultaneously minimized. In typical applications, regularization is only performed for the trainable parameters so that $L_t(x,\theta)\equiv L_t(\theta)$, but here we will consider the slightly more general case where it is also possible to regularize the state at each layer. In summary, we wish to solve the following problem
\begin{align}
  &\min_{\bftheta\in\boldsymbol{\Theta}} J(\bftheta)
  :=\frac{1}{S}\sum_{s=1}^{S} \Phi_s(x_{s,T}) 
  + \frac{1}{S}\sum_{s=1}^{S}\sum_{t=0}^{T-1} L_t(x_{s,t},\theta_t) \nonumber \\
  &\text{ subject to:} \nonumber \\
  &x_{s,t+1} = f_t(x_{s,t},\theta_t), \, t=0,\dots,T-1, \, s\in[S]
  \label{eq:control_problem}
\end{align}
where we have defined for shorthand $\boldsymbol{\Theta}:=\{\Theta_0\times\dots\times\Theta_{T-1}\}$ and $[S]:=\{1,\dots,S\}$. One may recognize problem~\eqref{eq:control_problem} as a classical fixed-time, variable-terminal-state optimal control problem in discrete time~\cite{ogata1995discrete}, in fact a special one with almost decoupled dynamics across samples in $[S]$. 

\subsection{The Pontryagin's Maximum Principle}
\label{sec:pmp}
Maximum principles of the Pontryagin type~\cite{boltyanskii1960theory,pontryagin1987mathematical} usually consist of necessary conditions for optimality in the form of the maximization of a certain Hamiltonian function. The distinguishing feature is that it does not assume differentiability (or even continuity) of $f_t$ with respect to $\theta$. Consequently the optimality condition and the algorithms based on it need not rely on gradient-descent type updates. This is an attractive feature for certain classes of applications. 

Let $\bftheta^*=\{\theta_0,\dots,\theta_{T-1}\}\in\boldsymbol{\Theta}$ be a solution of~\eqref{eq:control_problem}. We now outline informally the Pontryagin's maximum principle (PMP) that characterizes $\bftheta^*$. First, for each $t$ we define the Hamiltonian function 
$H_t:\R^{d_t}\times\R^{d_{t+1}}\times\Theta_t\rightarrow \R$ by
\begin{equation}
  H_t(x,p,\theta) := p\cdot f_t(x,\theta) - \tfrac{1}{S} L_t(x,\theta).
  \label{eq:hamil_def}
\end{equation}
One can show the following necessary conditions. 
\begin{theorem}[Discrete PMP, Informal Statement]
\label{thm:PMP}
  Let $f_t$ and $\Phi_s$, $s=1,\dots, S$ be sufficiently smooth in $x$.
  Assume further that for each $t$ and $x\in\R^{d_t}$, the sets $\{f_t(x,\theta):\theta\in\Theta_t\}$ and
  $\{L_t(x,\theta):\theta\in\Theta_t\}$ are convex. 
  Then, there exists co-state processes $\bfp_s^*:=\{p^*_{s,t}:t=0,\dots,T\}$, such that following holds 
  for $t=0,\dots,T-1$ and $s\in[S]$:
  \small
  \begin{align}
    &x^*_{s,t+1} = \nabla_p H_t(x^*_{s,t}, p^*_{s,t+1}, \theta^*_t),
    \quad x^*_{s,0} = x_{s,0} \label{eq:pmp_state}\\
    &p^*_{s,t} = \nabla_x H_t(x^*_{s,t}, p^*_{s,t+1}, \theta^*_{t}),
    \quad p^*_{s,T} = -\tfrac{1}{S}\nabla \Phi_s(x^*_{s,T}) \label{eq:pmp_costate}\\
    &\sum_{s=1}^{S} H_t(x^*_{s,t},p^*_{s,t+1},\theta^*_t) \geq 
    \sum_{s=1}^{S} H_t(x^*_{s,t},p^*_{s,t+1},\theta),
    \, \forall \theta\in\Theta_t \label{eq:pmp_max}
  \end{align}
  \normalsize
\end{theorem}
The full statement of Theorem~\ref{thm:PMP} involve explicit smoothness assumptions and additional technicalities (such as the inclusion of an abnormal multiplier). In Appendix A, we state these assumptions and give a sketch of its proof based on~\citet{halkin1966maximum}. 

Let us discuss the PMP in detail. The state equation~\eqref{eq:pmp_state} is simply the forward propagation equation~\eqref{eq:dyn_sys} under the optimal parameters $\bftheta^*$. Eq.~\eqref{eq:pmp_costate} defines the evolution of the co-state $\bfp_s^*$. To draw an analogy with nonlinear programming, the co-state can be interpreted as a set of Lagrange multipliers that enforces the constraint~\eqref{eq:dyn_sys} when the optimization problem~\eqref{eq:control_problem} is regarded as a joint optimization problem in $\bftheta$ and $\bfx_s$, $s\in[S]$. In the optimal control and PMP viewpoint, it is perhaps more appropriate to think of the dynamics~\eqref{eq:pmp_costate} as the evolution of the normal vector of a separating hyper-plane, which separates the set of reachable states and the set of states where the objective function takes on values smaller than the optimum (see Appendix A). 

The Hamiltonian maximization condition~\eqref{eq:pmp_max} is the centerpiece of the PMP. It says that an optimal solution $\bftheta^*$ must {\it globally maximize the (summed) Hamiltonian} for {\it each} layer $t=0,\dots,T-1$. 
Let us contrast this statement with usual first-order optimality conditions of the form $\nabla_{\bftheta} J(\bftheta^*) = 0$. A key observation is that in Theorem~\ref{thm:PMP}, we did not make reference to the derivative of any quantity with respect $\bftheta$. In fact, the PMP holds even if $f_t$ is not differentiable, or even continuous, with respect to $\theta$, as long as the convexity 
assumptions are satisfied. 
On the other hand, if we assume for each $t$: 1) $f_t$ is differentiable with respect to $\theta$; 2) $H_t$ is concave in $\theta$; and 3) $\theta^*_t$ lies in the interior of $\Theta_t$, then the Hamiltonian maximization condition~\eqref{eq:pmp_max} is equivalent to the condition $\nabla_\theta \sum_s H_t=0$ for all $t$, which one can then show is equivalent to $\nabla_{\bftheta} J=0$ (See Appendix C, proof of Prop. C.1). In other words, the PMP can be viewed as a stronger set of necessary conditions (at optimality, $H_t$ is not just stationary, but globally maximized) and has meaning in more general scenarios, e.g.~when stationarity with respect to $\bftheta$ is not achievable due to constraints, or not defined due to non-differentiability. 

\begin{remark}
\label{rem:singular}
It may occur that $\sum_s H_t(x^*_{s,t},p^*_{s,t+1},\theta)$ is constant for all $\theta\in\Theta_t$, in which
case the problem is singular~\cite{athans2013optimal}. In such cases, the PMP is trivially satisfied by any
$\theta$ and so the PMP does not tell us anything useful. This may arise especially in the case where there
are no regularization terms. 
\end{remark}

\subsection{The Convexity Assumption}
\label{sec:conv_assumption}
The most crucial assumption in Theorem~\ref{thm:PMP} is the convexity of the sets $\{f_t(x,\theta):\theta\in\Theta_t\}$ and $\{L_t(x,\theta):\theta\in\Theta_t\}$ for each fixed $x$~\footnote{Note that this is in general unrelated to the convexity, in the sense of functions, of $f_t$ with respect to either $x$ or $\theta$. For example, the scalar function $f(x,\theta)=\theta^3 \sin(x)$ is evidently non-convex in both arguments, but $\{f(x,\theta):\theta\in\R\}$ is convex for each $x$. On the other hand $\{\theta x:\theta\in \{-1,1\} \}$ is non-convex because of a non-convex admissable set.}.
We now discuss how restrictive these assumptions are with regard to deep neural networks. Let us first assume that the admissable sets $\Theta_t$ are convex. Then, the assumption with respect to $L_t$ is not restrictive since most regularizers (e.g.~$\ell_1, \ell_2$) satisfy it. 
Let us consider the convexity of $\{f_t(x,\theta):\theta\in\Theta_t\}$. 
In classical feed-forward neural networks, there are two types of layers: trainable ones and non-trainable ones. Suppose layer $t$ is non-trainable (e.g.~$f(x_t,\theta_t)=\sigma(x_t)$ where $\sigma$ is a non-linear activation function), then for each $x$ the set $\{f_t(x,\theta):\theta\in\Theta_t\}$ is a singleton, and hence trivially convex. On the other hand, in trainable layers $f_t$ is usually affine in $\theta$. This includes fully connected layers, convolution layers and batch normalization layers~\cite{ioffe2015batch}. In these cases, as long as the admissable set $\Theta_t$ is convex, we again satisfy the convexity assumption. Residual networks also satisfy the convexity constraint if one introduces auxillary variables (see Appendix A.1). When the set $\Theta_t$ is not convex, then it is in general not true that the PMP constitute necessary conditions. 

Finally, we remark that in the original derivation of the PMP for continuous-time control systems~\cite{boltyanskii1960theory} (i.e. $\dot{x}_{s,t}=f_t(x_{s,t},\theta_t), t\in[0,T]$ in place of Eq.~\eqref{eq:dyn_sys}), the convexity condition can be removed due to the ``convexifying'' effect of integration with respect to time~\cite{halkin1966maximum,warga1962relaxed}. Hence, the convexity condition is purely an artifact of discrete-time dynamical systems.

\section{The Method of Successive Approximations}
\label{sec:msa}
The PMP (Eq.~\eqref{eq:pmp_state}-\eqref{eq:pmp_max}) gives us a set of necessary conditions an optimal solution to~\eqref{eq:control_problem} must satisfy. However, it does not tell us how to find one such solution. The goal of this section is to discuss algorithms for solving~\eqref{eq:control_problem} based on the maximum principle. 

On closer inspection of Eq.~\eqref{eq:pmp_state}-\eqref{eq:pmp_max}, one can see that they each represent a manifold in solution space consisting of all possible $\bftheta$, $\{\bfx_s, s\in[S]\}$ and $\{\bfp_s, s\in[S]\}$, and the intersection of these three manifolds must contain an optimal solution, if one exists. Consequently, an iterative projection method that successively projects a guessed solution onto each of the manifolds is natural. This is the method of successive approximations (MSA), which was first introduced to solve continuous-time optimal control problems~\cite{krylov1962msa,chernousko1982method}. 
Let us now outline a discrete-time version. 

Start from an initial guess $\bftheta^0:=\{\theta^0_t,t=0,\dots,T-1\}$. For each sample $s$, we define 
$\bfx_s^{\bftheta^0}:=\{x^{\bftheta^0}_{s,t}: t=0,\dots,T\}$ by the dynamics
\begin{equation}
  x^{\bftheta^0}_{s,t+1} = f_t(x^{\bftheta^0}_{s,t},\theta^0_t), 
  \quad x^{\bftheta^0}_{s,0} = x_{s,0},
  \label{eq:msa_state}
\end{equation}
for $t=0,\dots,T-1$. Intuitively, this is a projection onto the manifold defined by Eq.~\eqref{eq:pmp_state}. Next, we perform the projection onto the manifold defined by Eq.~\eqref{eq:pmp_costate}, i.e.~we define $\bfp_s^{\bftheta^0}:=\{p^{\bftheta^0}_{s,t}: t=0,\dots,T\}$ by the backward dynamics
\begin{equation}
  p^{\bftheta^0}_{s,t} = \nabla_x H(x^{\bftheta^0}_{s,t},p^{\bftheta^0}_{s,t+1},\theta^0_t), 
  \quad p^{\bftheta^0}_{s,T} = -\tfrac{1}{S}\nabla \Phi_s(x^{\bftheta^0}_{s,T}),
  \label{eq:msa_costate}
\end{equation}
for $t=T-1,\dots,0$. Finally, we project onto manifold defined by Eq.~\eqref{eq:pmp_max} by performing Hamiltonian maximization to obtain $\bftheta^1:=\{\theta^1_t:t=0,\dots,T-1\}$ with
\begin{equation}
  \theta^1_{t} = \argmax_{\theta\in\Theta_t} \sum_{s=1}^{S} 
  H_t(x^{\bftheta^0}_{s,t}, p^{\bftheta^0}_{s,t+1}, \theta). 
  \quad t=0,\dots,T-1.
  \label{eq:msa_max}
\end{equation}
The steps~\eqref{eq:msa_state}-\eqref{eq:msa_max} are then repeated until convergence. We summarize the basic
MSA algorithm in Alg.~\ref{alg:basic_msa}. 

Let us contrast the MSA with gradient-descent based methods. Similar to the formulation of the PMP, at no point did we take the derivative of any quantity with respect to $\theta$. Hence, we can in principle apply this to problems that are not differentiable with respect to $\bftheta$. However, the catch is that the Hamiltonian maximization step~\eqref{eq:msa_max} may not be trivial to evaluate. Nevertheless, observe that the maximization step is {\it decoupled} across different layers of the neural network, and hence it is a much smaller problem than the original optimization problem, and its solution method can be parallelized. Alternatively, as seen in Sec.~\ref{sec:application}, one can exploit cases where the maximization step has explicit solutions. 

The basic MSA (Alg.~\ref{alg:basic_msa} can be shown to converge for problems where $f_t$ is linear and the costs $\Phi_s, L_t$ are quadratic~\cite{aleksandrov1968accumulation}. In general, however, unless a good initial condition is given, the MSA may diverge. Let us understand the nature of such phenomena by obtaining rigorous error estimates per-iteration of Eq.~\eqref{eq:msa_state}-\eqref{eq:msa_max}. 

\begin{algorithm}[tb]
  \caption{Basic MSA}
  \label{alg:basic_msa}
\begin{algorithmic}
  \STATE Initialize: $\bftheta^0=\{\theta^0_t\in\Theta_t:t=0\dots,T-1\}$;
  \small
  \FOR{$k = 0$ {\bfseries to} \#Iterations}
  \STATE $x^{\bftheta^k}_{s,t+1} = f_t(x^{\bftheta^k}_{s,t},\theta^k_t)$, $x^{\bftheta^k}_{s,0}=x_{s,0}$, $\forall s,t$;
  \STATE $p^{\bftheta^k}_{s,t} = \nabla_x H_t(x^{\bftheta^k}_{s,t}, p^{\bftheta^k}_{s,t+1},\theta^k_t), p^{\bftheta^k}_{s,T}=-\frac{1}{S}\nabla \Phi_s(x_{s,T})$, $\forall s,t$;
  % $x^{\bftheta^k}_{s,t+1} = f_t(x^{\bftheta^k}_{s,t},\theta^k_t)$ for $s=1,\dots,S$ and $t=0,\dots,T-1$ with $x^{\bftheta^k}_{s,0}=x_{s,0}$
  \STATE $\theta^{k+1}_t = \argmax_{\theta\in\Theta_t} \sum_{s=1}^{S}H_t(x^{\bftheta^k}_{s,t},p^{\bftheta^k}_{s,t+1},\theta)$ for $t=0,\dots,T-1$;
  \normalsize
  \ENDFOR
\end{algorithmic}
\end{algorithm}

\subsection{An Error Estimate for the MSA}
\label{sec:error_est_msa}
In this section, we derive a rigorous error estimate for the MSA, which can help us understand its dynamics. Let us define $W_t := \conv \{x\in\R^{d_t}: \exists \bftheta \text{ and } s \text{ s.t. } x^{\bftheta}_{s,t} = x \}$,
where $x^{\bftheta}_t$ is defined according to Eq.~\eqref{eq:msa_state}. This is the convex hull of all states reachable at layer $t$ by some initial sample and some choice of the values for the trainable parameters. 
Let us now make the following assumptions:
\begin{enumerate} 	
  \item[(A1)] $\Phi_s$ is twice continuously differentiable, with $\Phi_s$ and $\nabla \Phi_s$ satisfying a Lipschitz condition, i.e.~there exists $K>0$ such that for all $x,x'\in W_T$ and $s\in[S]$
  \small
  \begin{align*}
    \vert\Phi_s(x)-\Phi_s(x')\vert + \Vert\nabla \Phi_s(x)-\nabla \Phi_s(x')\Vert \leq K \Vert x-x' \Vert    
  \end{align*}
  \normalsize
  \item[(A2)] $f_t(\cdot,\theta)$ and $L_t(\cdot,\theta)$ are twice continuously differentiable in $x$, with $f_t,\nabla_x f_t,L_t,\nabla_x L_t$ satisfying Lipschitz conditions in $x$ uniformly in $t$ and $\theta$, i.e.~there exists $K>0$ such that
  \small
  \begin{align*}
    &\Vert f_t(x,\theta)-f_t(x',\theta)\Vert + \Vert \nabla_x f_t(x,\theta)-\nabla_x f_t(x',\theta)\Vert_2 \\ 
    &+ \vert L_t(x,\theta)-L_t(x',\theta)\vert 
    + \Vert \nabla_x L_t(x,\theta)-\nabla_x L_t(x',\theta)\Vert \\
    &\leq K \Vert x-x' \Vert
  \end{align*}
  \normalsize
  for all $x,x'\in W_t$, $\theta\in\Theta_t$ and $t=0,\dots,T-1$.
\end{enumerate}
 
Again, let us discuss these assumptions with respect to neural networks. 
Note that both assumptions are more easily satisfied if each $W_t$ is bounded, which is usually
implied by the boundedness of $\Theta_t$. Although this is not typically true in principle, we can safely assume this in practice by truncating weights that are too large in magnitude. 
Consequently, (A1) is not very restrictive, since many commonly employed loss functions (mean-square, soft-max with cross-entropy) satisfy these assumptions. 
In (A2), the regularity assumption on $L_t$ is again not an issue, because we mostly take $L_t$ to be independent of $x$. On the other hand, the regularity of $f_t$ with respect to $x$ is sometimes restrictive. 
For example, ReLU activations does not satisfy (A2) due to non-differentiability. Nevertheless, any suitably mollified version (like Soft-plus) does satisfy it. Moreover, tanh and sigmoid activations also satisfy (A2). 
Finally, unlike in Theorem~\ref{thm:PMP}, we do not assume the convexity of the sets $\{f_t(x,\theta):\theta\in\Theta_t\}$ and $\{L_t(x,\theta):\theta\in\Theta_t\}$, and hence the results in this section applies to discrete-weight neural networks considered in Sec.~\ref{sec:application}. 
With the above assumptions, we prove the following estimate.
\begin{theorem}[Error Estimate for Discrete MSA]
\label{thm:error_est}
  Let assumptions (A1) and (A2) be satisfied. Then, there exists a constant $C>0$, independent of $S$, $\bftheta$ and $\bfphi$, such that for any $\bftheta,\bfphi\in\boldsymbol{\Theta}$, we have
  \small
  \begin{align}
    &J(\bfphi) - J(\bftheta) \nonumber \\
    \leq& 
    -\sum_{t=0}^{T-1} \sum_{s=1}^{S} H_t(x^{\bftheta}_{s,t}, p^{\bftheta}_{s,t+1}, \phi_{t}) 
    - H_t(x^{\bftheta}_{s,t}, p^{\bftheta}_{s,t+1}, \theta_{t}) 
    \label{eq:error_est_max}\\ 
    &+ \frac{C}{S} \sum_{t=0}^{T-1} \sum_{s=1}^{S}
    \Vert f_t(x^{\bftheta}_{s,t},\phi_t) - f_t(x^{\bftheta}_{s,t},\theta_t) \Vert^2 
    \label{eq:error_est_state} \\ 
    % &+ C \sum_{t=0}^{T-1} \vert L_t(x^{\bftheta}_t,\theta_t) - L_t(x^{\bftheta}_t,\phi_t) \vert^2 
    % \label{eq:error_est_state} \\ 
    &+ \frac{C}{S} \sum_{t=0}^{T-1} \sum_{s=1}^{S}
    \Vert \nabla_x f_t(x^{\bftheta}_{s,t}, \phi_t) 
    - \nabla_x f_t(x^{\bftheta}_{s,t}, \theta_t) \Vert_2^2, 
    \label{eq:error_est_costate1} \\
    &+ \frac{C}{S} \sum_{t=0}^{T-1} \sum_{s=1}^{S}
    \Vert \nabla_x L_t(x^{\bftheta}_{s,t}, \phi_t) 
    - \nabla_x L_t(x^{\bftheta}_{s,t}, \theta_t) \Vert^2, 
    \label{eq:error_est_costate2}
  \end{align}
  \normalsize
  where $\bfx_s^{\bftheta}$, $\bfp_s^{\bftheta}$ are defined by Eq.~\eqref{eq:msa_state} and~\eqref{eq:msa_costate}. 
\end{theorem}
\begin{proof}
  The proof follows from elementary estimates and a discrete Gronwall's lemma. See Appendix B.
\end{proof}

Theorem~\ref{thm:error_est} relates the decrement of the total objective function $J$ with respect to the iterative projection steps of the MSA. Intuitively, Theorem~\ref{thm:error_est} says that the Hamiltonian maximization step~\eqref{eq:msa_max} is generally the right direction, because a large magnitude of~\eqref{eq:error_est_max} results in higher loss improvement. However, whenever we change the parameters from $\bftheta$ to $\bfphi$ (e.g.~during the maximization step~\eqref{eq:msa_max}), we incur non-negative penalty terms~\eqref{eq:error_est_state}-\eqref{eq:error_est_costate2}. Observe that these penalty terms vanish if $\bfphi=\bftheta$, or more generally, when the state and co-state equations (Eq.~\eqref{eq:msa_state},~\eqref{eq:msa_costate}) are still satisfied when $\bftheta$ is replaced by $\bfphi$. In other words, these terms measure the distance from manifolds defined by the state and co-state equations when the parameter changes.
Alg.~\ref{alg:basic_msa} diverges when these penalty terms dominate the gains from~\eqref{eq:error_est_max}. This insight can point us in the right direction of developing convergent modifications of the basic MSA. We shall now discuss this in the context of some specific applications.

\section{Neural Networks with Discrete Weights}
\label{sec:application}
We now turn to the application of the theory developed in the previous section on the MSA, which is a PMP-based numerical method for training deep neural networks. As discussed previously, the main strength of the PMP and MSA formalism is that we do not rely on gradient-descent type updates. This is particularly useful when one considers neural networks with (some) trainable parameters that can only take values in a discrete set. Then, any small gradient update to the parameters will almost always be infeasible. 
In this section, we will consider two such cases: {\it binary networks}, where weights are restricted to $\{-1,+1\}$; and {\it ternary networks}, where weights are selected from $\{-1,+1,0\}$. These networks are potentially useful for low-memory devices as storing the trained weights requires less memory. 
In this section, we will modify the MSA so that we can train these networks in a principled way. 

\subsection{Binary Networks}
\label{sec:binary}
Binary neural networks are those with binary trainable layers, e.g.~in the fully connected case,
\begin{equation}
  f_t(x,\theta) = \theta x
  \label{eq:binary_ft}
\end{equation}
where $\theta\in\Theta_t=\{-1,+1\}^{d_t\times d_{t+1}}$ is a binary matrix. A similar form of $f_t$ holds for convolution neural networks after reshaping, except that $\Theta_t$ is now the set of Toeplitz binary matrices. Hereafter, we will consider the fully connected case for simplicity of exposition. It is also natural to set the regularization to $0$ since there is in general no preference between $+1$ or $-1$. Thus, the Hamiltonian has the form
\[
  H_t(x,p,\theta) = p\cdot \theta x.
\]
Consequently, the Hamiltonian maximization step~\eqref{eq:msa_max} has explicit solution, given by
\[
  \argmax_{\theta\in\Theta_t} \sum_{s=1}^{S} H_t(x^{\bftheta^k}_{s,t}, p^{\bftheta^k}_{s,t+1},\theta)
  = \sgn(M^{\bftheta^k}_t)
\]
where $M^{\bftheta}_t:=\sum_{s=1}^{S} p^{\bftheta}_{s,t+1}(x^{\bftheta}_{s,t})^T$. 
Note that the sign function is applied element-wise. If $[M^{\bftheta}_t]_{ij}=0$, then the arg-max is arbitrary. 
Using Theorem~\ref{thm:error_est} with the form of $f_t$ given by~\eqref{eq:binary_ft} and the fact that $L_t\equiv 0$, we get
\begin{align*}
  J(\bfphi) - J(\bftheta) 
  \leq & - \sum_{t=0}^{T-1}\sum_{s=1}^{S} H_t(x^{\bftheta^k}_{s,t}, p^{\bftheta^k}_{s,t+1},\theta) \\
  & + \frac{C}{S} \sum_{t=0}^{T-1}(1 + \sum_{s=1}^{S} \Vert x^{\bftheta}_{s,t} \Vert^2) \Vert\phi_t - \theta_t \Vert_F^2, \\
\end{align*}
Note that we have used the inequality $\Vert\cdot\Vert_2\leq\Vert\cdot\Vert_F$. Assuming that $\Vert x^\theta_{s,t} \Vert$ is $\mathcal{O}(1)$, 
we may then decrease $J$ by not only maximizing the Hamiltonian, but also penalizing the difference
$\Vert\phi_t - \theta_t \Vert_F$, i.e.~for each $k$ and $t$ we set
\begin{align}
  \theta^{k+1}_t =& \argmax_{\theta\in\Theta_t} \left[ \sum_{s=1}^{S} H_t(x^{\bftheta^k}_{s,t}, p^{\bftheta^k}_{s,t+1},\theta) 
  - \rho_{k,t} \Vert \theta - \theta^k \Vert^2_F \right] 
  \label{eq:augmented_H}
\end{align}
for some penalization parameters $\rho_{k,t}>0$. This again has the explicit solution
\begin{align}
  [\theta^{k+1}_t]_{ij}=&
  \begin{cases}
    \sgn([M^{\bftheta^k}_t]_{ij}) & \vert [M^{\bftheta^k}_t]_{ij} \vert \geq 2\rho_{k,t}\\
    [\theta^k_t]_{ij} & \text{otherwise}
  \end{cases}
  \label{eq:msa_argmax_regularized_binary}
\end{align}
Therefore, we simply replace the parameter update step in Alg.~\ref{alg:basic_msa} with~\eqref{eq:msa_argmax_regularized_binary}. Furthermore, to deal with mini-batches, we keep a moving average of $M^{\bftheta^k}_t$ across different mini-batches and use the averaged value to update our parameters. It is found empirically that this further stabilizes the algorithm. 
Note that the assumption $\Vert x^{\bftheta}_{s,t} \Vert$ is $\mathcal{O}(1)$ can be achieved by normalization, e.g. batch-normalization~\cite{ioffe2015batch}. We summarize the algorithm in Alg.~\ref{alg:binary_msa}. Further algorithmic details are found in Appendix D, where we also discuss the choice of hyper-parameters and the convergence of the algorithm for a simple binary regression problem. A rigorous proof of convergence in the general case is beyond the scope of this work, but we demonstrate via experiments below that the algorithm performs well on the tested benchmarks. 

\begin{algorithm}[tb]
  \caption{Binary MSA}
\label{alg:binary_msa}
\begin{algorithmic}
  \small
  \STATE Initialize: $\bftheta^0$, $\overline{\mathbf{M}}^0$;
  \STATE Hyper-parameters: $\rho_{k,t}$, $\alpha_{k,t}$;
  \FOR{$k = 0$ {\bfseries to} \#Iterations}
  \STATE $x^{\bftheta^k}_{s,t+1} = f_t(x^{\bftheta^k}_{s,t},\theta^k_t) \qquad \forall s,t$
  \STATE $\qquad$ with $x^{\bftheta^k}_{s,0}=x_{s,0}$;
  \STATE $p^{\bftheta^k}_{s,t} = \nabla_x H_t(x^{\bftheta^k}_{s,t},p^{\bftheta^k}_{s,t+1},\theta^k_t) \qquad \forall s,t$
  \STATE $\qquad$ with $p^{\bftheta^k}_{s,T}=-\tfrac{1}{S}\nabla \Phi_s(x_{s,T})$;
  \STATE $\overline{M}^{k+1}_t = \alpha_{k,t}\overline{M}^{k}_t + (1-\alpha_{k,t})\sum_{s=1}^S p^{\bftheta^k}_{s,t+1}(x^{\bftheta^k}_{s,t})^T$
  \STATE $[\theta^{k+1}_t]_{ij}=
          \begin{cases}
            \sgn([\overline{M}^{k+1}_t]_{ij}) & \vert [\overline{M}^{k+1}_t]_{ij} \vert \geq 2\rho_{k,t}\\
            [\theta^k_t]_{ij} & \text{otherwise}
          \end{cases}$
  \STATE $\forall t$, $i$, and $j$;
  % \STATE $\theta^{k+1}_t = \argmax_{\theta\in\Theta_t} \sum_{s=1}^{S}H_t(x^{\bftheta^k}_{s,t},p^{\bftheta^k}_{s,t+1},\theta)$ for $t=0,\dots,T-1$;
  \normalsize
  \ENDFOR
\end{algorithmic}
\end{algorithm}
 
We apply Alg.~\ref{alg:binary_msa} to train binary neural networks on various benchmark datasets and compare the results from previous work on training binary-weight neural networks~\cite{courbariaux2015binaryconnect}. 
We consider a fully-connected neural network on MNIST~\cite{lecun1998mnist}, as well as (shallow) convolutional networks on CIFAR-10~\cite{krizhevsky2009learning} and SVHN~\cite{netzer2011reading} datasets. The network structures are mostly identical to those considered in~\citet{courbariaux2015binaryconnect} for ease of comparison. Complete implementation and model details are found in Appendix D. The graphs of training/testing loss and error rates are are shown in Fig.~\ref{fig:binary}. We observe that our algorithm performs well in terms of an optimization algorithm, as measured by the training loss and error rates. For the harder datasets (CIFAR-10 and SVHN), we have rapid convergence but worse test loss and error rates at the end, possibly due to overfitting. We note that in~\cite{courbariaux2015binaryconnect}, many regularization strategies are performed. We expect that similar techniques must be employed to improve the testing performance. However, these issues are out of the scope of the optimization framework of this paper. Note that we also compared the results of BinaryConnect without regularization strategies such as stochastic binarization, but the results are similar in that our algorithm converges very fast with very low training losses, but sometimes overfits. 

\begin{figure}[tb]
  \centering
  \includegraphics[width=8cm]{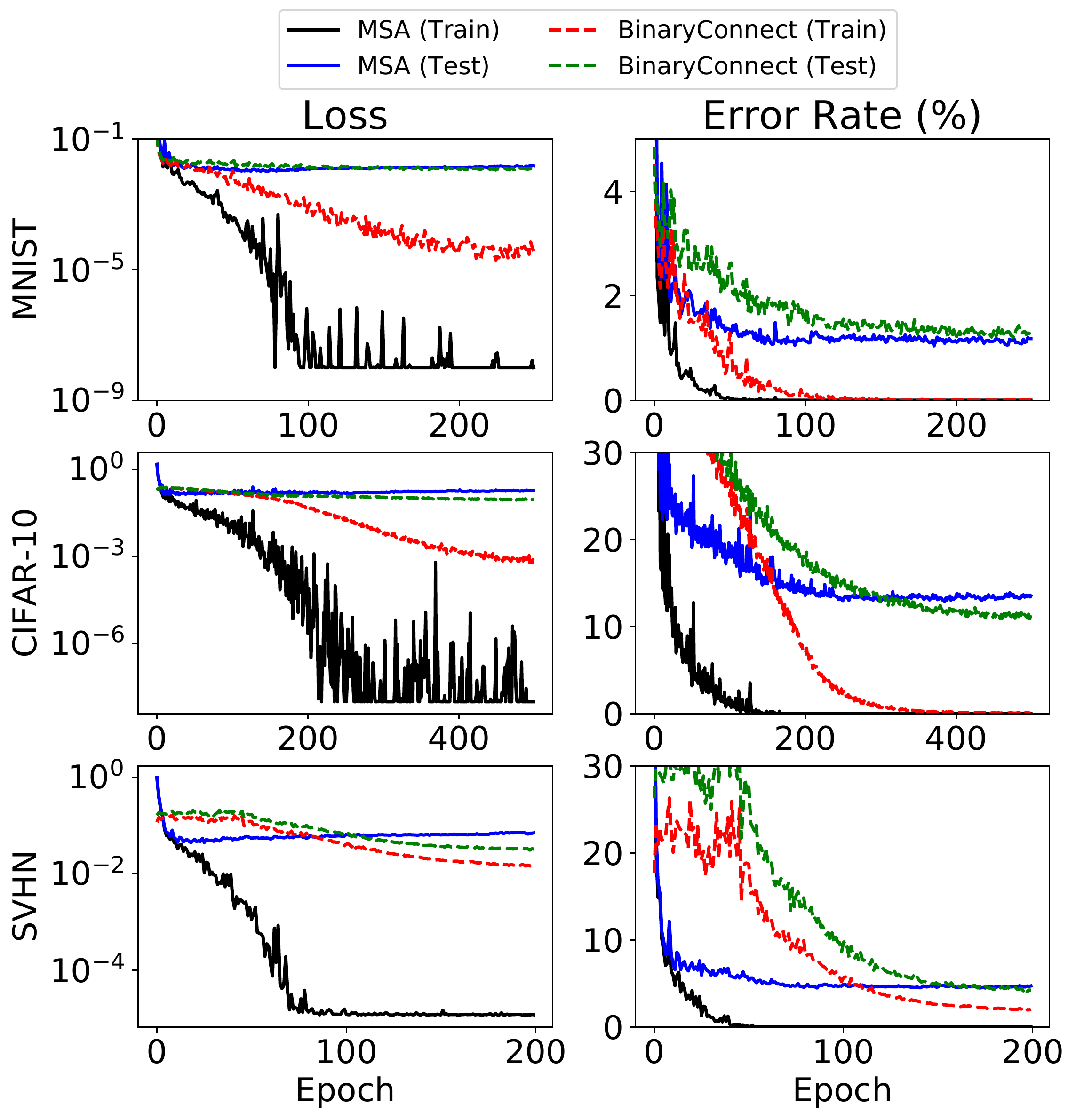}
  \caption{Comparison of binary MSA (Alg.~\ref{alg:binary_msa}) with BinaryConnect~\cite{courbariaux2015binaryconnect} (with binary variables for inference). We observe that MSA has good convergence in terms of the training loss and error rates, showing that it is an efficient optimization algorithm. Note that to avoid broken lines, when the loss equals 0 exactly, we replace it by 1e-8 on the log-scale. The test loss for the bigger datasets (CIFAR10, SVHN) eventually becomes worse due to overfitting, hence some regularization techniques is needed for applications which are prone to overfitting. }
\label{fig:binary}
\end{figure} 

\subsection{Ternary Networks}
\label{sec:ternary}
We shall consider another case where the network weights are allowed to take on values in $\{-1,+1,0\}$. In this case, our goal is to explore the sparsification of the network. To this end, we shall take $L_t(x,\theta)=\lambda_t\Vert\theta \Vert_F^2$ for some parameter $\lambda_t$. Note that since the weights are restricted to the ternary set, any component-wise $\ell_p$ regularization for $p>0$ are identical. The higher the $\lambda_t$ values, the more sparse the solution will be. 

As in Sec.~\ref{sec:binary}, we can write down the Hamiltonian for a fully connected ternary layer as
\[
  H_t(x,p,\theta) = p\cdot \theta x - \tfrac{1}{S}\lambda_t\Vert \theta\Vert_F^2.
\]
The derivation of the ternary algorithm then follows directly from those in Sec.~\ref{sec:binary}, but with the new form of Hamiltonian above and that $\Theta_t=\{-1,+1,0\}^{d_t\times d_{t+1}}$. Maximizing the augmented Hamiltonian~\eqref{eq:augmented_H} with $H_t$ as defined above, we obtain the ternary update rule
\begin{equation}
  [\theta^{k+1}_t]_{ij} = 
    \begin{cases}
      +1 & [M^{\bftheta^k}_t]_{ij} \geq  \rho_{k,t} (1-2[\theta^k_t]_{ij}) + \lambda_t \\
      -1 & [M^{\bftheta^k}_t]_{ij} \leq -\rho_{k,t} (1+2[\theta^k_t]_{ij}) - \lambda_t \\
      0 & \text{otherwise.}
    \end{cases}
  \label{eq:msa_argmax_regularized_ternary}
\end{equation}
We replace the parameter update step in Alg.~\ref{alg:binary_msa} by~\eqref{eq:msa_argmax_regularized_ternary} to obtain the MSA algorithm for ternary networks. For completeness, we give the full ternary algorithm in Alg.~\ref{alg:ternary_msa}. We now test the ternary algorithm on the same benchmarks used in Sec.~\ref{sec:binary} and the results are shown in Fig.~\ref{fig:ternary}. Observe that the performance on training and testing datasets is similar to the binary case (Fig.~\ref{fig:binary}), but the ternary networks achieve high degrees of sparsity in the weights, with only 0.5-2.5\% of the trained weights being non-zero, depending on the dataset. This potentially offers significant memory savings compared to its binary or full floating precision counterparts. 

\begin{algorithm}[tb]
  \caption{Ternary MSA}
\label{alg:ternary_msa}
\begin{algorithmic}
  \small
  \STATE Initialize: $\bftheta^0$, $\overline{\mathbf{M}}^0$;
  \STATE Hyper-parameters: $\rho_{k,t}$, $\alpha_{k,t}$;
  \FOR{$k = 0$ {\bfseries to} \#Iterations}
  \STATE $x^{\bftheta^k}_{s,t+1} = f_t(x^{\bftheta^k}_{s,t},\theta^k_t) \qquad \forall s,t$
  \STATE $\qquad$ with $x^{\bftheta^k}_{s,0}=x_{s,0}$;
  \STATE $p^{\bftheta^k}_{s,t} = \nabla_x H_t(x^{\bftheta^k}_{s,t},p^{\bftheta^k}_{s,t+1},\theta^k_t) \qquad \forall s,t$
  \STATE $\qquad$ with $p^{\bftheta^k}_{s,T}=-\tfrac{1}{S}\nabla \Phi_s(x_{s,T})$;  \STATE $\overline{M}^{k+1}_t = \alpha_{k,t}\overline{M}^{k}_t + (1-\alpha_{k,t})\sum_{s=1}^S p^{\bftheta^k}_{s,t+1}(x^{\bftheta^k}_{s,t})^T$
  \STATE $[\theta^{k+1}_t]_{ij} = 
  \begin{cases}
    +1 & [\overline{M}^{k+1}_{t}]_{ij} \geq  \rho_{k,t} (1-2[\theta^k_t]_{ij}) + \lambda_t \\
    -1 & [\overline{M}^{k+1}_{t}]_{ij} \leq -\rho_{k,t} (1+2[\theta^k_t]_{ij}) - \lambda_t \\
    0 & \text{otherwise.}
  \end{cases}$
  \STATE $\forall t$, $i$, and $j$; 
  % \STATE $\theta^{k+1}_t = \argmax_{\theta\in\Theta_t} \sum_{s=1}^{S}H_t(x^{\bftheta^k}_{s,t},p^{\bftheta^k}_{s,t+1},\theta)$ for $t=0,\dots,T-1$;
  \normalsize
  \ENDFOR
\end{algorithmic}
\end{algorithm}

\begin{figure}[tb]
  \centering
  \includegraphics[width=8cm]{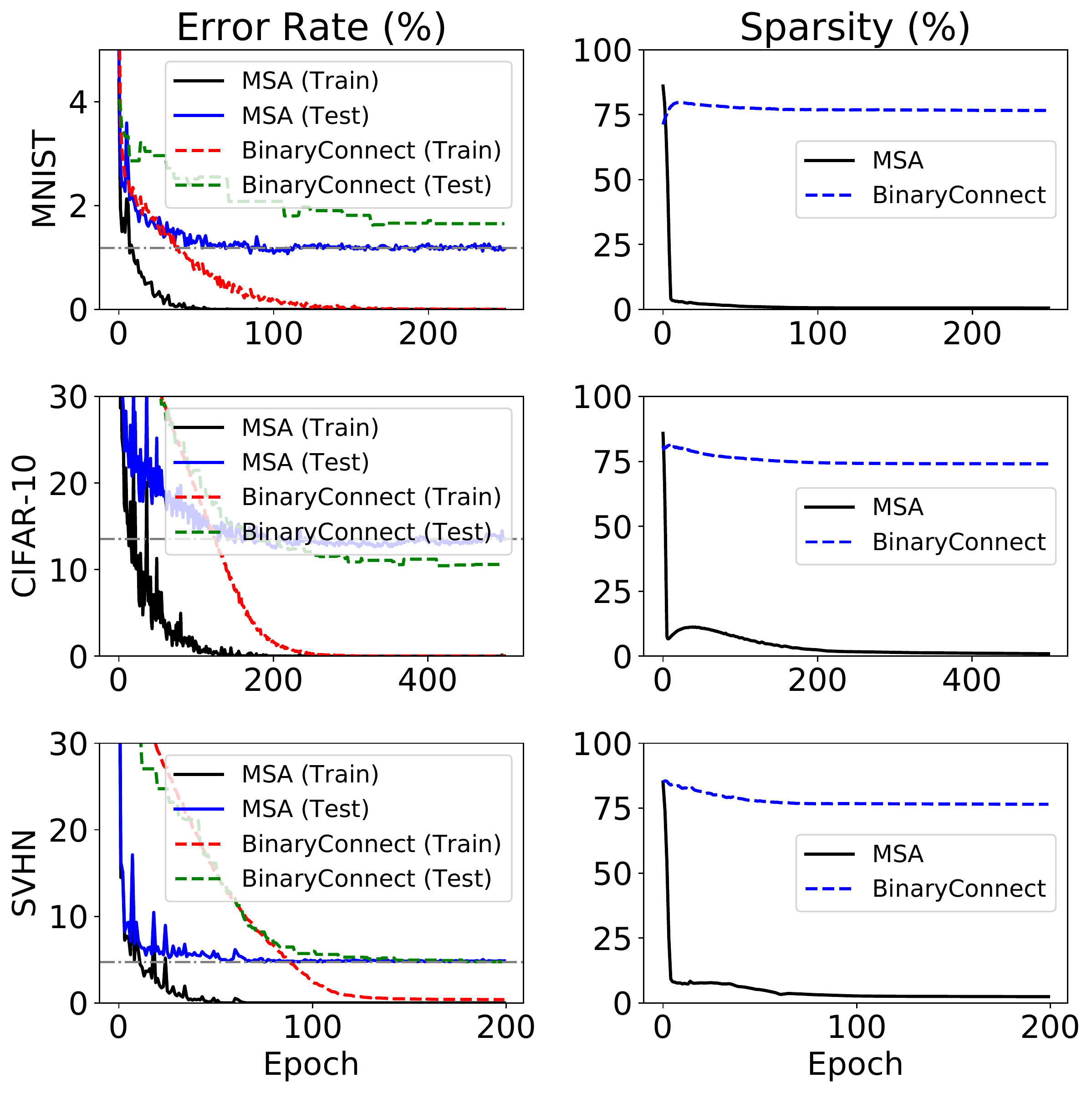}
  \caption{Performance of the ternary MSA (Alg.~\ref{alg:ternary_msa}) vs. BinaryConnect with a simple thresholding procedure described in~\citep{li2016ternary}. 
  In the second column of plots, we show the sparsity of the networks (defined as the percentage of all weights that are non-zero) as training proceeds. Observe that the MSA algorithm finds solutions with comparable error rates with the binary case (whose final test-error is plotted as a grey horizontal line) but are very sparse. In comparison, the simple thresholding of BinaryConnect does not produce sparse solutions. In fact, the final sparsities for MSA are approximately: MNIST:~$<$1.0\%; CIFAR-10:~0.9\%; SVHN:~2.4\%. It is expected that sparser solutions can be found by adjusting the penalty parameters $\lambda_t$.}
\label{fig:ternary}
\end{figure}

\section{Discussion and Related Work}
\label{sec:discussion}
We begin with a discussion of the results presented thus far. We first introduced the viewpoint that deep learning can be regarded as a discrete-time optimal control problem. Consequently, an important result in optimal control theory, the Pontryagin's maximum principle, can be applied to give a set of necessary conditions for optimality. These are in general stronger conditions than the usual optimality conditions based on the vanishing of first-order partial derivatives. Moreover, they apply to broader contexts such as problems with constraints on the trainable parameters or problems that are non-differentiable in the trainable parameters. 
However, we note that specific assumptions regarding the convexity of some sets must be satisfied. We showed that they are justified for conventional neural networks, but not necessarily so for all neural networks (e.g.~binary, ternary networks). 

Next, based on the PMP, we introduced an iterative projection technique, the discrete method of successive approximations (MSA), to find an optimal solution of the learning problem. A rigorous error estimate (Theorem~\ref{thm:error_est}) is derived for the discrete MSA, which can be used to both understand its dynamics and to derive useful algorithms. This should be viewed as the main theoretical result of the present paper. 
Note that the usual back-propagation with gradient descent can be regarded as a simple modification of the MSA, if differentiability conditions are assumed (see Appendix C). Nevertheless, we note that Theorem~\ref{thm:error_est} itself does not assume any regularity conditions with respect to the trainable parameters. Moreover,neither does it require the convexity conditions in Theorem~\ref{thm:PMP}, and hence applies to a wider range of neural networks, including those in Sec.~\ref{sec:application}. All results up to this point apply to general neural networks (assuming that the respective conditions are satisfied), and are not specific to the applications presented subsequently. 

In the last part of this work, we apply our results to devise training algorithms for discrete-weight neural networks, i.e.~those with trainable parameters that can only take values in a discrete set. Besides potential applications in model deployment in low-memory devices, the main reasons for choosing such applications are two-fold. First, gradient-descent updates are not applicable by itself because small updates to parameters are prohibited by the discrete equality constraint on the trainable parameters. However, our method based on the MSA is applicable since it does not perform gradient-descent updates. Second, in such applications the potentially expensive Hamiltonian maximization steps in the MSA have explicit solutions. This makes MSA an attractive optimization method for problems of this nature. In Sec~\ref{sec:application}, we demonstrate the effectiveness of our methods on various benchmark datasets. Interestingly, the ternary network exhibits extremely sparse weights that perform almost as well as its binary counter-part (see Fig.~\ref{fig:ternary}). Also, the phenomena of overfitting in Fig.~\ref{fig:binary} and~\ref{fig:ternary} are interesting as overfitting is generally less common in stochastic gradient based optimization approaches. This seems to suggest that the MSA based methods optimize neural networks in a rather different way. 

Let us now put our work in the context of the existing literature. First, the optimal control approach we adopt is quite different from the prevailing viewpoint of nonlinear programming~\cite{bertsekas1999nonlinear,bazaraa2013nonlinear,kuhn2014nonlinear} and the analysis of the derived gradient-based algorithms~\cite{moulines2011non,shamir2013stochastic,bach2013non,xiao2014proximal,shalev2014accelerated} for the training of deep neural networks. In particular, the PMP (Thm.~\ref{thm:PMP}) and the MSA error estimate (Thm.~\ref{thm:error_est}) do not assume differentiability and do not characterize optimality via gradients (or sub-gradients) with respect to trainable parameters. In this sense, it is a stronger and more robust condition, albeit sometimes requiring different assumptions. The optimal control and dynamical systems viewpoint has been discussed in the context of deep learning in~\citet{e2017proposal,li2017maximum} and dynamical systems based discretization schemes has been introduced in~\citet{haber2017stable,chang2017reversible}. Most of these works have theoretical basis in continuous-time dynamical systems. In particular,~\citet{li2017maximum} analyzed continuous-time analogues of neural networks in the optimal control framework and derived MSA-based algorithms in continuous time. In contrast, the present work presents a discrete-time formulation, which is natural in the usual context of deep learning. The discrete PMP turns out to be more subtle, as it requires additional assumptions of convexity of reachable sets (Thm.~\ref{thm:PMP}). Note also that unlike the estimates derived in~\citet{li2017maximum}, Thm.~\ref{thm:error_est} holds rigorously for discrete-time neural networks. 
The present method for stabilizing the MSA is also different from that in~\citet{li2017maximum}, where augmented Lagrangian type of modifications are employed. The latter would not be effective here because weights cannot be updated infinitesimally without violating the binary/ternary constraint. Moreover, the present methods that rely on explicit solutions of Hamiltonian maximization are fast (comparable to SGD) on a wall-clock basis. 

In the deep learning literature, the connection between optimal control and deep learning has been qualitative discussed in~\citet{le1988theoretical} and applied to the development of automatic differentiation and back-propagation~\cite{bryson1975applied,baydin2015automatic}. However, there are relatively fewer works relating optimal control algorithms to training neural networks beyond the classical gradient-descent with back-propagation. Optimal control based strategies in hyper-parameter tuning has been discussed in~\citet{li2017stochastic}. 

In the continuous-time setting, the Pontryagin's maximum principle and the method of successive approximations have a long history, with a large body of relevant literature including, but not limited to~\citet{boltyanskii1960theory,pontryagin1987mathematical,bryson1975applied,bertsekas1995dynamic,athans2013optimal,krylov1962msa,aleksandrov1968accumulation,krylov1972algorithm,chernousko1982method,lyubushin1982modifications}. The discrete-time PMP have been studied in~\citet{halkin1966maximum,holtzman1966convexity,holtzman1966discretional,holtzman1966maximum,canon1970theory}, where Theorem~\ref{thm:PMP} and its extensions are proved. To the best of our knowledge, the discrete-time MSA and its quantitative analysis have not been performed in either the deep learning or the optimal control literature. 

Sec.~\ref{sec:application} concerns the application of the MSA, in particular Thm.~\ref{thm:error_est}, to develop training algorithms for binary and ternary neural networks. There are a number of prior work exploring the training of similar neural networks, such as~\citet{courbariaux2015binaryconnect,hubara2016binarized,rastegari2016xnor,tang2017train,li2016ternary,zhu2016trained}. Theoretical analysis for the case of convex loss functions is carried out in~\citet{li2017training}. Our point of numerical comparison for the binary MSA algorithm is~\citet{courbariaux2015binaryconnect}, where optimization of binary networks is based on shadow variables with full floating-point precision that is iteratively truncated to obtain gradients. We showed in Sec.~\ref{sec:binary} that the binary MSA is competitive as a training algorithm, but is in need of modifications to reduce overfitting for certain datasets.
Training ternary networks has been discussed in~\citet{hwang1967discrete,kim2014x1000,li2016ternary,zhu2016trained}. The difference in our ternary formulation is that we explore the sparsification of networks using a regularization parameter. In this sense it is related to compression of neural networks (e.g.~\citet{han2015deep}), but our approach trains a network that is naturally ternary, and compression is achieved during training by a regularization term. 
Generally, a contrasting aspect of our approach from the aforementioned literature is that the theory of optimal control, together with Theorem.~\ref{thm:error_est}, provide a theoretical basis for the development of our algorithms. Nevertheless, further work is required to rigorously establish the convergence of these algorithms. 
We also mention a recent work~\cite{yin2018binaryrelax} which analyzes quantized networks and develops algorithms based on relaxing the discrete-weight constraints into continuous regularizers. Lastly, there are also analyses of quantized networks from a statistical-mechanical viewpoint~\cite{baldassi2015subdominant,baldassi2016unreasonable,baldassi2016learning,baldassi2017role}.

\section{Conclusion and Outlook}
\label{sec:conclusion}
In this paper, we have introduced the discrete-time optimal control viewpoint of deep learning. In particular, the PMP and the MSA form an alternative theoretical and algorithmic basis for deep learning that may apply to broader contexts. As an application of our framework, we considered the training of binary and ternary neural networks, in which we develop effective algorithms based on optimal control.

There are certainly many avenues of future work. An interesting mathematical question is the applicability of the PMP for discrete-weight neural networks, which does not satisfy the convexity assumptions in Theorem~\ref{thm:PMP}. It will be desirable to find the condition under which rigorous statements can be made. Another question is to establish the convergence of the algorithms presented. 

% Finally, it will be useful to test our algorithms on bigger datasets and devise principled regularization techniques to reduce the overfitting phenomenon observed in Sec.~\ref{sec:application}.

\pagebreak

\onecolumn

\part*{Appendix}
\appendix

\section{Full Statement and Sketch of the Proof of Theorem~1}
\label{sec:proof_thm_1}
In this section, we give the full statement of Theorem~1 and a sketch of its proof as presented in~\cite{halkin1966maximum}. We note that in~\cite{halkin1966maximum}, more general initial and final conditions are considered. For simplicity, we shall stick to the current formulation in the main text. We note also that the result presented here has been extended (in the sense that the convexity condition has been relaxed to directional convexity)~\cite{holtzman1966convexity,holtzman1966discretional} and proven in different ways subsequently~\cite{canon1970theory}. 

Before we begin, we simplify the notation by concatenating all the samples $x_s$ into a large vector $x=(x_1,\dots,x_S)$. The functions $f_t$ are then redefined accordingly in the natural way. Moreover, we define the total loss function $\Phi(x):=\tfrac{1}{S}\sum_s \Phi_s(x_s)$ and the total regularization $L_t(x,\theta)=\tfrac{1}{S}\sum_s L_t(x_s,\theta)$. Consequently, we have the reformulated problem
\begin{align}
  &\min_{\bftheta\in\boldsymbol{\Theta}} J(\bftheta)
  := \Phi(x_{T}) 
  + \sum_{t=0}^{T-1} L_t(x_{t},\theta_t) \nonumber \\
  &\text{ subject to:} \nonumber \\
  &x_{t+1} = f_t(x_{t},\theta_t), \quad t=0,\dots,T-1.
  \label{eq:control_problem_reformulated}
\end{align}
We now make the following assumptions:
\begin{itemize}
  \item[(B1)] $\Phi$ is twice continuous differentiable.
  \item[(B2)] $f_t(\cdot,\theta),L_t(\cdot,\theta)$ are twice continuously differentiable with respect to $x$, and $f_t(\cdot,\theta),L_t(\cdot,\theta)$ together with their $x$ partial derivatives are uniformly bounded in $t$ and $\theta$.
  \item[(B3)] The sets $\{ f_t(x,\theta): \theta\in\Theta_t \}$ and $\{ L_t(x,\theta): \theta\in\Theta_t \}$ are convex for every $t$ and $x\in\R^{d_t}$.
\end{itemize}

The full statement of Theorem~1 is as follows:

\begin{theorem}[Discrete PMP, Full Statement]
\label{thm:PMP_full}
  Let (B1)-(B3) be satisfied. Suppose that $\bftheta^*:=\{\theta^*_t : t=0,\dots,T-1\}$
  is an optimal solution of~\eqref{eq:control_problem_reformulated} and $\bfx^*:=\{x^*_t:t=0,\dots,T\}$ is the corresponding state process with $\bftheta=\bftheta^*$. 
  Then, there exists a co-state (or adjoint) process $\bfp^*:=\{p^*_t:t=0,\dots,T\}$ and a real number $\beta\geq 0$ (abnormal multiplier) such that $\{\bfp^*,\beta \}$ are not all zero, and the following holds:
  \begin{align}
    &x^*_{t+1} = \nabla_p H_t(x^*_t, p^*_{t+1}, \theta^*_t) &
    &x^*_0 = x_0 \label{eq:full_pmp_state}\\
    &p^*_{t} = \nabla_x H_t(x^*_t, p^*_{t+1}, \theta^*_t) &
    &p^*_T = -\beta \nabla \Phi(x^*_T) \label{eq:full_pmp_costate}\\
    &H_t(x^*_t,p^*_t,\theta^*_t) \geq H_t(x^*_t, p^*_t, \theta) 
    &&\text{ for all } \theta\in\Theta_t 
    \label{eq:full_pmp_max}
  \end{align}
  for $t=0,1,\dots,T-1$, where the Hamiltonian function $H$ is defined as
  \[
    H_t(x,p,\theta) := p\cdot f_t(x,\theta) - \beta L_t(x,\theta).
  \]
\end{theorem}

\begin{remark}
  Compared with the informal statement, the full statement involves an abnormal multiplier $\beta$. It exists to cover degenerate cases. This is related to ``normality'' in the calculus of variations~\cite{bliss1938normality}, or constraint qualification in the language of nonlinear programming~\cite{kuhn2014nonlinear}. When it equals $0$, the problem is degenerate. In applications we often focus on non-degenerate cases where $\beta$ is positive, in which case we can normalize $\{p^*_t,\beta \}$ accordingly so that $\beta=1$. We then obtain the informal statement in the main text. 
\end{remark}

\begin{proof}[Sketch of the proof of Theorem~\ref{thm:PMP_full}]
  To begin with, we may assume without loss of generality that $L\equiv0$. To see why this is so, we define an extra scalar variable $w_t$ with
  \[
     w_{t+1}=w_{t}+L_t(x_t, \theta_{t}), \quad w_0=0.
  \]
  We then append $w$ to $x$ to form the new $(d_t+1)$-dimensional state vector $(x,w)$. Accordingly, we modify $f_t(x,\theta)$ to $(f_t(x,\theta), w + L_t(x,\theta))$ and $\Phi(x)$ to $\Phi(x) + w$. It is clear that all assumptions (B1)-(B3) are preserved. 

  As in the main text, we define the set of reachable states by the original dynamical system 
  \begin{equation}
    W_t := \{x\in\R^{d_t}: \exists \bftheta
    \text{ s.t. } x^{\bftheta}_t = x \}
  \end{equation}
  where $x^{\bftheta}_t$ is the evolution of the dynamical system for $x_t$ under $\bftheta$. This is basically the set of all states that the system can reach under ``some'' control at time $t$. 
  Let $\{\bfx^*,\bftheta^*\}$ be a pair of optimal solutions of~\eqref{eq:control_problem_reformulated}. Let us define the set of all final states with lower loss value than the optimum as
  \begin{equation}
    S := \{x\in\R^{d_T}: \Phi(x) < \Phi(x^*_T)  \}.
  \end{equation}
  Then, it is clear that $W_T$ and $B$ are disjoint. Otherwise, $\{\bfx^*,\bftheta^*\}$ would not have been optimal. Now, if $W_T$ and $B$ are convex, then one can then use separation properties of convex sets to prove the theorem. However, in general they are non-convex (even if (B3) is satisfied). The idea is to consider the following linearized problem
  \begin{align}
    &\psi_{t+1} = f_t(x^*_t,\theta_t) + \nabla_x f_t(x^*_t,\theta^*_t)(\psi_t-x^*_t), \quad t=0,1,\dots,T-1 \nonumber \\
    &\psi_0 = x_0
    \label{eq:control_problem_linearized}
  \end{align}
  Then, we can similarly define the counter-parts to $W_t$ and $S$ as
  \begin{equation}
    W^+_t := \{x\in\R^{d_t}: \exists \bftheta
    \text{ s.t. } \psi^{\bftheta}_t = x \}
  \end{equation}
  and 
  \begin{equation}
    S^+ := \{x\in\R^{d_T}: (x - x^*_T)\cdot \nabla \Phi(x^*_T) < 0  \}.
  \end{equation}
  It is clear that the sets $W^+_T$ and $S^+$ are both convex. In~\cite{halkin1966maximum}, the author proves an important linearization lemma that says: {\bf if $W_T$ and $S$ are disjoint, then $W^+_T$ and $S^+$
  are separated}, i.e.~there exists a non-zero vector $\pi \in \R^{d_T}$ such that
  \begin{align}
    &(x-x^*_T)\cdot \pi \leq 0  & &x\in W_T^+ \label{eq:separation_1}\\
    &(x-x^*_T)\cdot \pi \geq 0  & &x\in S^+ \label{eq:separation_2}
  \end{align}
  Here, $\pi$ is the normal of a separating hyper-plane of the convex sets $W^+_T$ and $S^+$. In fact, one can show that $\pi=-\beta \nabla \Phi(x^*_T)$ for some $\beta\geq 0$. We note here that the linearization lemma, i.e.~the separation of $W^+_T$ and $S^+$, forms the bulk of the proof of the theorem in~\cite{halkin1966maximum}. The proof relies on topological properties of non-separated convex sets. We shall omit its proof here and refer the reader to~\cite{halkin1966maximum}. 

  Now, we may define $p^*_T = \pi$, and for $t\leq T$, set
  \begin{equation}
    p^*_t = \nabla_x H_t(x^*_t, p^*_{t+1}, \theta^*_t) = \nabla_x {f(x^*_t, \theta^*_t)}^T p^*_{t+1}.
    \label{eq:adjoint_equation}
  \end{equation}
  In other words, $p^*_t$ evolves the normal $\pi$ of the separating hyper-plane of $W_T^+$ and $S^+$ backwards in time. An important property one can check is that $p^*_t$ and $\psi_t$ (defined by Eq.~\eqref{eq:adjoint_equation} and~\eqref{eq:control_problem_linearized}) are adjoint of each other at the optimum, i.e.~if $\theta_t=\theta^*_t$, then we have
  \begin{equation}
    (\psi_{t+1}-x^*_{t+1})\cdot p^*_{t+1} = (\psi_{t}-x^*_{t})\cdot p^*_{t}.
    \label{eq:adjoint_relation}
  \end{equation}
  This fact allows one to prove the Hamiltonian maximization condition~\eqref{eq:full_pmp_max}. Indeed,
  suppose that for some $t\in\{0,\dots,T-1\}$ the condition is violated, i.e.~there exists $\tilde{\theta}\in \Theta_t$ such that
  \[
    H_t(x^*_t, p^*_{t+1}, \tilde{\theta}) = H_t(x^*_t, p^*_{t+1}, \theta^*_t) + \epsilon
  \]
  for some $\epsilon>0$. This means
  \[
    p^*_{t+1}\cdot f_t(x^*_t, \tilde{\theta}) = p^*_{t+1}\cdot f_t(x^*_t, \theta^*_t) + \epsilon
  \]
  i.e.,
  \[
    p^*_{t+1}\cdot (f_t(x^*_t, \tilde{\theta}) - x^*_{t+1}) = \epsilon
  \]
  Now, we simply evolve $\psi_s$, $s\geq t+1$ with $\theta_s=\theta^*_s$ but the initial condition 
  $\psi_{t+1} = f_t(x^*_t, \tilde{\theta})$. Then, Eq.~\eqref{eq:adjoint_relation} implies that $\pi\cdot (\psi_T - x^*_T)\cdot = \epsilon > 0$, but this contradicts~\eqref{eq:separation_1}. 
\end{proof}

\begin{remark}
  Note that in the original proof~\cite{halkin1966maximum}, it is also assumed that $\nabla_x f_t$ is non-singular, which also forces $d_t=d$ to be constant for all $t$. This is obviously not satisfied naturally by most neural networks that have changing dimensions. However, 
  one can check that this condition only serves to ensure that if $p^*_T\neq 0$, then $p^*_t\neq 0$ for all 
  $t=0,\dots,T-1$. Hence, without this assumption, we can only be sure that not all $\{\bfp^*,\beta\}$ are 0. 
\end{remark}

\subsection{The Convexity Condition for Neural Networks}
\label{eq:conv_condition_for_nns}
As also discussed in the main text, the most stringent condition in Theorem~\ref{thm:PMP_full} is the convexity condition for $f_t$, i.e.~the set $\{f_t(x,\theta):\theta\in\Theta\}$ must be convex. It is easy to see that for the usual feed-forward neural networks, one can decompose it in such a way that the convexity constraint is satisfied as long as the parameter sets $\Theta_t$ are convex. Indeed,we have
\[
  x_{t+1} = \sigma(g_t(x_t, \theta_t))
\]
where $\sigma$ is some non-trainable nonlinear activation function and $g_t$ is affine in $\theta$. We can simply decompose this into two steps
\begin{align*}
  &x'_{t+1} = g_t(x_t, \theta_t),\\
  &x'_{t+2} = \sigma(x'_{t+1}).
\end{align*}
Then, $x'_{t+2} = x_{t+1}$ but each of these two steps now satisfy the convexity constraint. 

Similarly, in residual networks, we can usually write the layer transformation as
\[
  x_{t+1} = x_{t} + h_t(\sigma(g_t(x_t, \theta_t)), \phi_t)
\]
where $g_t,h_t$ are maps affine in $\theta$ and $\phi$ respectively, and $\sigma$ is a non-trainable non-linearity. The above cannot be straightforwardly split into two layers as there is a shortcut connection from $x_n$. However, we can introduce auxillary variables $y_t$ and consider the 3-step decomposition
\begin{align*}
  &x'_{t+1} = g_t(x_t, \theta_t) & &y'_{t+1} = x_{t},  \\
  &x'_{t+2} = \sigma(x'_{t+1}) & &y'_{t+2} = y'_{t+1},  \\
  &x'_{t+3} = y'_{t+2} + h_t(x'_{t+2}, \phi_{t}) & &y'_{t+3} = y'_{t+2}.
\end{align*}
It is clear then that $x'_{t+3}$ is equal to $x_{t+1}$ in the residual network layer. Furthermore, this new decomposed system satisfy the convexity assumption as long as $\Theta_t$ is a convex set.

\section{Proof of Theorem~2}
\label{sec:proof_thm_2}
In this section, we prove Theorem~2 in the main text using some elementary estimates. 
Let us first prove a useful result. 

\begin{lemma}[Discrete Gronwall's Lemma]
\label{lem:gronwall}
Let $K\geq 0$ and $u_{t}$, $w_{t}$, be non-negative real valued sequences satisfying
\[
  u_{t+1} \leq K u_{t} + w_{t},
\]
for $t=0,\dots,T-1$. Then, we have for all $t=0,\dots,T$,
\[
  u_{t} \leq \max(1,K^T) \left( u_0 + \sum_{s=0}^{T-1} w_{s} \right).
\]
\end{lemma}
\begin{proof}
  We prove by induction the inequality 
  \begin{equation}
    u_{t} \leq \max(1,K^t) \left( u_0 + \sum_{s=0}^{t-1} w_{s} \right),
    \label{eq:gronwall_1}
  \end{equation}
  from which the lemma follows immediately. The case $t=0$ is trivial. 
  Suppose the above is true for some $t$, we have
  \begin{align*}
    u_{t+1} &\leq K u_{t} + w_{t} \\
    & \leq K \max(1,K^t) \left( u_0 + \sum_{s=0}^{t-1} w_{s} \right) + w_t \\
    & \leq \max(1,K^{t+1}) \left( u_0 + \sum_{s=0}^{t-1} w_{s} \right) + \max(1,K^{t+1}) w_t \\
    & = \max(1,K^{t+1}) \left( u_0 + \sum_{s=0}^{t} w_{s} \right).
  \end{align*}
  This proves~\eqref{eq:gronwall_1} and hence the lemma. 
\end{proof}

Let us now commence the proof of a preliminary lemma that estimates the magnitude of $\bfp_s^{\bftheta}$ for any $\bftheta\in\boldsymbol{\Theta}$. 
Hereafter, $C$ will be stand for any generic constant that does not depend on $\bftheta,\bfphi$ and $S$ (batch size), but may depend on other fixed quantities such as $T$ and the Lipschitz constants $K$ in (A1)-(A2). Also, the value of $C$ is allowed to change to another constant value with the same dependencies from line to line in order to reduce notational clutter. 

\begin{lemma}
\label{lem:p_bound}
  There exists a constant $C>0$ such that for each
  $t=0,\dots,T$ and $\bftheta\in\boldsymbol{\Theta}$, we have
  \[
    \Vert p^{\bftheta}_{s,t}\Vert\leq \frac{C}{S}.
  \]
  for all $s=1,\dots,S$. 
\end{lemma}
\begin{proof}
  First, notice that $p_{s,T}^{\bftheta}=-\frac{1}{S}\nabla\Phi_s(x_{s,T}^{\bftheta})$
  and so by assumption (A1), we have
  \[
  \Vert p_{s,T}^{\bftheta}\Vert=\frac{1}{S}\Vert\nabla\Phi_s(x_{s,T}^{\bftheta})\Vert\leq \frac{K}{S}.
  \]
  Now, for each $0\leq t<T$, we have by Eq. (8) and assumption (A2) in the main text, 
  \begin{align*}
    \Vert p_{s,t}^{\bftheta}\Vert
    = & \Vert\nabla_{x}H_t(x_{s,t}^{\bftheta},p_{s,t+1}^{\bftheta},\theta_{t})\Vert \\
    \leq & \Vert {\nabla_{x}f_t(x_{s,t}^{\bftheta},\theta_{t})}^T p_{s,t+1}^{\bftheta}\Vert 
    + \frac{1}{S} \nabla_{x} \Vert L_t(x_{s,t}^{\bftheta},\theta_t) \Vert \\
    \leq & K \Vert p_{s,t+1}^{\bftheta}\Vert
    + \frac{K}{S} 
  \end{align*}
  Using Lemma~\ref{lem:gronwall} with $t\rightarrow T-t$, we get
  \[
    \Vert p_{s,t}^{\bftheta}\Vert \leq \max(1,K^T)( \frac{K}{S} + \frac{TK}{S} ) = \frac{C}{S}.
  \]
\end{proof}

% Now, for $\theta,\phi\in\boldsymbol{\Theta}$, we denote the change in Hamiltonian
% due to the substitution of $\theta$ by $\phi$ at layer $t$ by 
% \[
% \Delta H_{\theta,\phi}(t):=H_t(x_{t}^{\bftheta},p_{t+1}^{\bftheta},\phi_{t})-H_t(x_{t}^{\bftheta},p_{t+1}^{\bftheta},\theta_{t})
% \]
% We now prove the following estimate. 

We are now ready to prove Theorem~2. 

\begin{proof}[Proof of Theorem~2]
Recall the definition
\[
  H_t(x,p,\theta) = p\cdot f_t(x,\theta) - \frac{1}{S}L_t(x,\theta).
\]
Let us define the quantity
\[
  I(\bfx,\bfp,\bftheta)
  :=\sum_{t=0}^{T-1}p_{t+1}\cdot x_{t+1}-H_t(x_{t},p_{t+1},\theta_{t}) - L_t(x_{t},\theta_{t})
\]
Then, from Eq. (7) from the main text, we know that $I(\bfx_s^{\bftheta},\bfp_s^{\bftheta},\bftheta)=0$
for any $s=1,\dots,S$ and $\bftheta\in\boldsymbol{\Theta}$. 
Let us now fix some sample $s$ and obtain corresponding estimates. We have
\begin{align}
  0 
  = & I(\bfx_s^{\bfphi},\bfp_s^{\bfphi},\bfphi)-I(\bfx_s^{\bftheta},\bfp_s^{\bftheta},\bftheta)\nonumber \\
  = & \sum_{t=0}^{T-1}
  p_{s,t+1}^{\bfphi}\cdot x_{s,t+1}^{\bfphi}-p_{s,t+1}^{\bftheta}\cdot x_{s,t+1}^{\bftheta}\nonumber \\
  & -\frac{1}{S}\sum_{t=0}^{T-1}
  L_t(x_{s,t}^{\bfphi},\phi_{t})-L_t(x_{s,t}^{\bftheta},\theta_{t})\nonumber \\
  & -\sum_{t=0}^{T-1}
  H_t(x_{s,t}^{\bfphi},p_{s,t+1}^{\bfphi},\phi_{t})-H_t(x_{s,t}^{\bftheta},p_{s,t+1}^{\bftheta},\phi_{t})\label{eq:lem_equality}
\end{align}
We can rewrite the first term on the right hand side as
\begin{align}
  &\sum_{t=0}^{T-1}p_{s,t+1}^{\bfphi}\cdot x_{s,t+1}^{\bfphi}-p_{s,t+1}^{\bftheta}\cdot x_{s,t+1}^{\bftheta} \nonumber\\
  =&\sum_{t=0}^{T-1}p_{s,t+1}^{\bftheta}\cdot\delta x_{s,t+1}
  +x_{s,t+1}^{\bftheta} \cdot \delta p_{s,t+1}
  +\delta x_{s,t+1} \cdot \delta p_{s,t+1},
  \label{eq:lem_est_0}
\end{align}
where we have defined $\delta x_{s,t}:=x_{s,t}^{\bfphi}-x_{s,t}^{\bftheta}$
and $\delta p_{s,t}:=p_{s,t}^{\bfphi}-p_{s,t}^{\bftheta}$. We may simplify further
by observing that $\delta x_{s,0}=0$, and so
\begin{align*}
  \sum_{t=0}^{T-1} p_{s,t+1}^{\bftheta}\cdot\delta x_{s,t+1}+x_{s,t+1}^{\bftheta} \cdot \delta p_{s,t+1} 
  = & p_{s,T}^{\bftheta}\cdot\delta x_{s,T}
  +\sum_{t=0}^{T-1}p_{s,t}^{\bftheta}\cdot\delta x_{s,t}+x_{s,t+1}^{\bftheta} \cdot \delta p_{s,t+1}\\
  = & p_{s,T}^{\bftheta}\cdot\delta x_{s,T}
  +\sum_{t=0}^{T-1}\nabla_{x}H_t(x_{s,t}^{\bftheta},p_{s,t+1}^{\bftheta},\theta_{t})\cdot\delta x_{s,t}\\
  & +\sum_{t=0}^{T-1}\nabla_{p}H_t(x_{s,t}^{\bftheta},p_{s,t+1}^{\bftheta},\theta_{t})\cdot\delta p_{s,t+1}
\end{align*}
By defining the extended vector $z_{s,t}^{\bftheta}:=(x_{s,t}^{\bftheta},p_{s,t+1}^{\bftheta})$,
we can rewrite this as
\begin{align}
\sum_{t=0}^{T-1} p_{s,t+1}^{\bftheta}\cdot\delta x_{s,t+1}+x_{s,t+1}^{\bftheta} \cdot \delta p_{s,t+1} 
= & p_{s,T}^{\bftheta}\cdot\delta x_{s,T} +\sum_{t=0}^{T-1}\nabla_{z}H_t(z_{s,t}^{\bftheta},\theta_{t})\cdot\delta z_{s,t}\label{eq:lem_est_1}
\end{align}
Similarly, we also have
\begin{align}
\sum_{t=0}^{T-1}\delta x_{s,t+1}\cdot\delta p_{s,t+1} = & \frac{1}{2}\sum_{t=0}^{T-1}\delta x_{s,t+1}\cdot\delta p_{s,t+1}+\frac{1}{2}\sum_{t=0}^{T-1}\delta x_{s,t+1}\cdot\delta p_{s,t+1}\nonumber \\
  = & \frac{1}{2}\delta x_{s,T}\cdot\delta p_{s,T}\nonumber \\
  & +\frac{1}{2}\sum_{t=0}^{T-1}(\nabla_{z}H_t(z_{s,t}^{\bfphi},\phi_{t})-\nabla_{z}H_t(z_{s,t}^{\bftheta},\theta_{t}))\cdot\delta z_{s,t}\nonumber \\
  = & \frac{1}{2}\delta x_{s,T}\cdot\delta p_{s,T}\nonumber \\
  & +\frac{1}{2}\sum_{t=0}^{T-1}(\nabla_{z}H_t(z_{s,t}^{\bftheta},\phi_{t})-\nabla_{z}H_t(z_{s,t}^{\bftheta},\theta_{t}))\cdot\delta z_{s,t}\nonumber \\
  & +\frac{1}{2}\sum_{t=0}^{T-1}\delta z_{s,t}\cdot\nabla_{z}^{2}H_t(z_{s,t}^{\bftheta}+r_{1}(t)\delta z_{s,t},\phi_{t})\delta z_{s,t}\label{eq:lem_est_2}
\end{align}
where in the last line we used Taylor's theorem with $r_{1}(t)\in\left[0,1\right]$
for each $t$. Now, we can rewrite the terminal terms (i.e. $T$ terms) in (\ref{eq:lem_est_1})
and (\ref{eq:lem_est_2}) as follows:
\begin{align}
  &(p_{s,T}^{\bftheta}+\frac{1}{2}\delta p_{s,T})\cdot\delta x_{s,T} \nonumber \\
  = & -\frac{1}{S}\nabla\Phi_s(x_{s,T}^{\bftheta})\cdot\delta x_{s,T}
  -\frac{1}{2S}(\nabla\Phi_s(x_{s,T}^{\bfphi})-\nabla\Phi_s(x_{s,T}^{\bftheta}))\cdot\delta x_{s,T}\nonumber \\
  = & -\frac{1}{S}\nabla\Phi_s(x_{s,T}^{\bftheta})\cdot\delta x_{s,T}-
  \frac{1}{2S}\delta x_{s,T}\cdot\nabla^{2}\Phi_s(x_{s,T}^{\bftheta}+r_{2}\delta x_{s,T})\delta x_{s,T}\nonumber \\
  = & -\frac{1}{S}(\Phi_s(x^{\bfphi}_T)-\Phi_s(x^{\bftheta}_T)) -\frac{1}{2S}\delta x_{s,T}\cdot[\nabla^{2}\Phi_s(x_{s,T}^{\bftheta}+r_{2}\delta x_{s,T}) 
  +\nabla^{2}\Phi_s(x_{s,T}^{\bftheta}+r_{3}\delta x_{s,T})]\delta x_{s,T}
  \label{eq:lem_est_3}
\end{align}
for some $r_{2},r_{3}\in[0,1]$. Lastly, for each $t=0,1,\dots,T-1$ we have
\begin{align}
  H_t(z_{s,t}^{\bfphi},\phi_{t})-H_t(z_{s,t}^{\bftheta},\theta_{t}) 
  = & H_t(z_{s,t}^{\bftheta},\phi_{t})-H_t(z_{s,t}^{\bftheta},\theta_{t})\nonumber \\
  & +\nabla_{z}H_t(z_{s,t}^{\bftheta},\phi_{t})\cdot\delta z_{s,t}\nonumber \\
  & +\frac{1}{2}\delta z_{s,t}\cdot\nabla_{z}^{2}H_t(z_{s,t}^{\bftheta}+r_{4}(t)\delta z_{s,t},\phi_{t})\delta z_{s,t}
  \label{eq:lem_est_4}
\end{align}
where $r_{4}(t)\in[0,1]$. 
  
Substituting Eq.~(\ref{eq:lem_est_0},~\ref{eq:lem_est_1},~\ref{eq:lem_est_2},~\ref{eq:lem_est_3},~\ref{eq:lem_est_4})
into Eq.~\eqref{eq:lem_equality} yields
\begin{align}
  &\frac{1}{S}\left[\Phi_s(x^{\bfphi}_{s,T})+\sum_{t=0}^{T-1}L_t(x^{\bfphi}_{s,t},\phi_{t})\right] 
  - \frac{1}{S}\left[\Phi_s(x^{\bftheta}_{s,T})+\sum_{t=0}^{T-1}L_t(x^{\bftheta}_{s,t},\theta_{t})\right] \nonumber \\
  = & -\sum_{t=0}^{T-1} H_t(x^{\bftheta}_t,p^{\bftheta}_{t+1},\phi_{t}) - H_t(x^{\bftheta}_t,p^{\bftheta}_{t+1},\theta_{t}) \nonumber \\
  % & -\sum_{t=0}^{T-1}\Delta H_{\phi,\theta}(t)\nonumber \\
  & +\frac{1}{2S}\delta x_{s,T}\cdot(\nabla^{2}\Phi_s(x_{s,T}^{\bftheta}+r_{2}\delta x_{s,T})+\nabla^{2}\Phi_s(x_{s,T}^{\bftheta}+r_{3}\delta x_{s,T}))\delta x_{s,T}\nonumber \\
  & +\frac{1}{2}\sum_{t=0}^{T-1}(\nabla_{z}H_t(z_{s,t}^{\bftheta},\phi_{t})-\nabla_{z}H_t(z_{s,t}^{\bftheta},\theta_{t}))\cdot\delta z_{s,t}\nonumber \\
  & +\frac{1}{2}\sum_{t=0}^{T-1}\delta z_{s,t}\cdot(\nabla_{z}^{2}H_t(z_{s,t}^{\bftheta}+r_{1}(t)\delta z_{s,t},\phi_{t})-\nabla_{z}^{2}H_t(z_{s,t}^{\bftheta}+r_{4}(t)\delta z_{s,t},\phi_{t}))\delta z_{s,t}
  \label{eq:lem_equality_subbed}
\end{align}
Note that by summing over all $s$, the left hand side is simply $J(\bfphi)-J(\bftheta)$. 
Let us further simplify the right hand side. 
First, by (A1), we have
\begin{equation}
  \delta x_{s,T}\cdot(\nabla^{2}\Phi_s(x_{s,T}^{\bftheta}+r_{2}\delta x_{s,T})+\nabla^{2}\Phi_s(x_{s,T}^{\bftheta}+r_{3}\delta x_{s,T}))\delta x_{s,T}
  \leq K \Vert \delta x_{s,T} \Vert^2.
  \label{eq:thm_2_rhs_1}
\end{equation}
Next, 
\begin{align}
  &(\nabla_{z}H_t(z_{s,t}^{\bftheta},\phi_{t})-\nabla_{z}H_t(z_{s,t}^{\bftheta},\theta_{t}))\cdot\delta z_{s,t} \nonumber \\
  \leq & \Vert \nabla_x H_t(x_{s,t}^{\bftheta},p_{s,t+1}^{\bftheta},\phi_{t}) - \nabla_x H_t(x_{s,t}^{\bftheta},p_{s,t+1}^{\bftheta},\theta_{t}) \Vert \Vert \delta x_{s,t} \Vert \nonumber \\
  & + \Vert \nabla_p H_t(x_{s,t}^{\bftheta},p_{s,t+1}^{\bftheta},\phi_{t}) - \nabla_p H_t(x_{s,t}^{\bftheta},p_{s,t+1}^{\bftheta},\theta_{t}) \Vert \Vert \delta p_{s,t+1} \Vert \nonumber \\
  \leq & \frac{1}{2S} \Vert \delta x_{s,t} \Vert^2 + \frac{S}{2} \Vert \nabla_x H_t(x_{s,t}^{\bftheta},p_{s,t+1}^{\bftheta},\phi_{t}) - \nabla_x H_t(x_{s,t}^{\bftheta},p_{s,t+1}^{\bftheta},\theta_{t}) \Vert^2 \nonumber \\
  & + \frac{S}{2} \Vert \delta p_{s,t} \Vert^2 + \frac{1}{2S} \Vert \nabla_p H_t(x_{s,t}^{\bftheta},p_{s,t+1}^{\bftheta},\phi_{t}) - \nabla_p H_t(x_{s,t}^{\bftheta},p_{s,t+1}^{\bftheta},\theta_{t}) \Vert^2 \nonumber \\
  \leq & \frac{1}{2S} \Vert \delta x_{s,t} \Vert^2 + \frac{C^2}{2S} \Vert \nabla_x f_t(x_{s,t}^{\bftheta},\phi_{t}) - \nabla_x f_t(x_{s,t}^{\bftheta},\theta_{t}) \Vert^2 \nonumber \\
  & + \frac{1}{2S} \Vert \nabla_x L_t(x_{s,t}^{\bftheta},\phi_{t}) - \nabla_x L_t(x_{s,t}^{\bftheta},\theta_{t}) \Vert^2 \nonumber \\
  & + \frac{S}{2} \Vert \delta p_{s,t} \Vert^2 + \frac{1}{2S} \Vert f_t(x_{s,t}^{\bftheta},\phi_{t}) -  f_t(x_{s,t}^{\bftheta},\theta_{t}) \Vert^2,
  \label{eq:thm_2_rhs_2}
  % \leq & \Vert \nabla_x f_t(x_{s,t}^{\bftheta},\phi_{t})-\nabla_x f_t(x_{s,t}^{\bftheta},\theta_{t}) \Vert_2 \Vert \delta x_{s,t} \Vert
  % \Vert p^{\bftheta}_{s,t+1} \Vert \\
  % & + \\
  % & + \Vert f_t(x_{s,t}^{\bftheta},\phi_{t})-f_t(x_{s,t}^{\bftheta},\theta_{t})\Vert 
\end{align}
where in the last line we have used Lemma~\ref{lem:p_bound}. Similarly, we can simplify the last term in~\eqref{eq:lem_equality_subbed}. Notice that the second derivative of $H_t$ with respect to $p$ vanishes since it is linear. Hence, as in Eq.~\eqref{eq:thm_2_rhs_1} and using Lemma~\ref{lem:p_bound}, we have
\begin{align}
  & \delta z_{s,t}\cdot(\nabla_{z}^{2}H_t(z_{s,t}^{\bftheta}+r_{1}(t)\delta z_{s,t},\phi_{t})-\nabla_{z}^{2}H_t(z_{s,t}^{\bftheta}+r_{4}(t)\delta z_{s,t},\phi_{t}))\delta z_{s,t} \nonumber \\
  \leq & \frac{2KC}{S} \Vert \delta x_{s,t} \Vert^2 + 4K \Vert \delta x_{s,t} \Vert \Vert \delta p_{s,t+1} \Vert \nonumber \\
  \leq & \frac{2KC}{S} \Vert \delta x_{s,t} \Vert^2 + \frac{2K}{S} \Vert \delta x_{s,t} \Vert^2
  + 2KS \Vert \delta p_{s,t+1} \Vert^2
  \label{eq:thm_2_rhs_3}
\end{align}
Substituting Eq.~(\ref{eq:thm_2_rhs_1},\ref{eq:thm_2_rhs_2},\ref{eq:thm_2_rhs_3}) into~\eqref{eq:lem_equality_subbed} and summing over $s$, we have (renaming constants)
\begin{align}
  &\frac{1}{S}\left[\Phi_s(x^{\bfphi}_{s,T})+\sum_{t=0}^{T-1}L_t(x^{\bfphi}_{s,t},\phi_{t})\right] 
  - \frac{1}{S}\left[\Phi_s(x^{\bftheta}_{s,T})+\sum_{t=0}^{T-1}L_t(x^{\bftheta}_{s,t},\theta_{t})\right] \nonumber \\
  = & -\sum_{t=0}^{T-1}
  H_t(x^{\bftheta}_t,p^{\bftheta}_{t+1},\phi_{t}) - H_t(x^{\bftheta}_t,p^{\bftheta}_{t+1},\theta_{t}) \nonumber \\
  & +\frac{C}{S} \sum_{t=0}^{T} \Vert \delta x_{s,t} \Vert^2
  +CS \sum_{t=0}^{T-1} \Vert \delta p_{s,t+1} \Vert^2 \nonumber \\ 
  & + \frac{C}{S} \sum_{t=0}^{T-1}
  \Vert f_t(x_{s,t}^{\bftheta},\phi_{t}) -  f_t(x_{s,t}^{\bftheta},\theta_{t}) \Vert^2 \nonumber \\
  & + \frac{C}{S} \sum_{t=0}^{T-1}
  \Vert \nabla_x f_t(x_{s,t}^{\bftheta},\phi_{t}) - 
  \nabla_x f_t(x_{s,t}^{\bftheta},\theta_{t}) \Vert^2 \nonumber \\
  & + \frac{C}{S} \sum_{t=0}^{T-1}
  \Vert \nabla_x L_t(x_{s,t}^{\bftheta},\phi_{t}) - \nabla_x L_t(x_{s,t}^{\bftheta},\theta_{t}) \Vert^2
  \label{eq:lem_equality_subbed_simplified}
\end{align}
It remains to estimate the magnitudes of $\delta x_{s,t}$ and $\delta p_{s,t}$. Observe that $\delta x_{s,0}=0$, hence we have for each $t=0,\dots,T-1$
\begin{align}
  \Vert \delta x_{s,t+1} \Vert 
  \leq & \Vert f_t(x_{s,t}^{\bfphi},\phi_{t}) -  f_t(x_{s,t}^{\bftheta},\phi_{t}) \Vert
  + \Vert f_t(x_{s,t}^{\bftheta},\phi_{t}) -  f_t(x_{s,t}^{\bftheta},\theta_{t}) \Vert \nonumber \\
  \leq & K \Vert \delta x_{s,t} \Vert + \Vert f_t(x_{s,t}^{\bftheta},\phi_{t}) -  f_t(x_{s,t}^{\bftheta},\theta_{t}) \Vert \nonumber
\end{align}
Using Lemma~\ref{lem:gronwall}, we have
\begin{equation}
  \Vert \delta x_{s,t} \Vert \leq C \sum_{t=0}^{T-1} \Vert f_t(x_{s,t}^{\bftheta},\phi_{t}) -  f_t(x_{s,t}^{\bftheta},\theta_{t}) \Vert
  \label{eq:thm_2_dx_est}
\end{equation}
Similarly, 
\begin{align}
  \Vert \delta p_{s,t} \Vert 
  \leq & \Vert \nabla_x H_t(x_{s,t}^{\bfphi},p_{s,t+1}^{\bfphi},\phi_{t})
  - \nabla_x H_t(x_{s,t}^{\bftheta},p_{s,t+1}^{\bftheta},\theta_{t}) \Vert \nonumber \\
  % + \Vert f_t(x_{s,t}^{\bftheta},\phi_{t}) -  f_t(x_{s,t}^{\bftheta},\theta_{t}) \Vert \nonumber \\
  \leq & 2K \Vert \delta p_{s,t+1} \Vert + \frac{C}{S} \Vert \delta x_{s,t} \Vert \nonumber \\
  & + \frac{C}{S} \Vert \nabla_x f_t(x_{s,t}^{\bftheta},\phi_{t}) -  \nabla_x f_t(x_{s,t}^{\bftheta},\theta_{t}) \Vert_2 \nonumber \\
  & + \frac{C}{S} \Vert \nabla_x L_t(x_{s,t}^{\bftheta},\phi_{t}) -  \nabla_x L_t(x_{s,t}^{\bftheta},\theta_{t}) \Vert, \nonumber
  % \Vert \nabla_x f_t(x_{s,t}^{\bfphi},\phi_{t}) - \nabla_x f_t(x_{s,t}^{\bftheta},\phi_{t}) \Vert
  % + \Vert f_t(x_{s,t}^{\bftheta},\phi_{t}) -  f_t(x_{s,t}^{\bftheta},\theta_{t}) \Vert \nonumber \\
  % \leq & K \Vert \delta x_{s,t} \Vert + \Vert f_t(x_{s,t}^{\bftheta},\phi_{t}) -  f_t(x_{s,t}^{\bftheta},\theta_{t}) \Vert \nonumber
\end{align}
and so by Lemma~\ref{lem:gronwall}, Eq.~\eqref{eq:thm_2_dx_est} and the fact that 
$\Vert \delta p_{T,s}\Vert \leq \tfrac{K}{S} \Vert\delta x_{T,s}\Vert$ (by (A1)), we have
\begin{align}
  \Vert \delta p_{s,t} \Vert \leq & \frac{C}{S} \sum_{t=0}^{T-1} \Vert f_t(x_{s,t}^{\bftheta},\phi_{t}) -  f_t(x_{s,t}^{\bftheta},\theta_{t}) \Vert \nonumber \\
  & + \frac{C}{S} \sum_{t=0}^{T-1} \Vert \nabla_x f_t(x_{s,t}^{\bftheta},\phi_{t}) - \nabla_x f_t(x_{s,t}^{\bftheta},\theta_{t}) \Vert_2 \nonumber \\
  & + \frac{C}{S} \sum_{t=0}^{T-1} \Vert \nabla_x L_t(x_{s,t}^{\bftheta},\phi_{t}) - \nabla_x L_t(x_{s,t}^{\bftheta},\theta_{t}) \Vert.
  \label{eq:thm_2_dp_est}
\end{align}
Finally, we conclude the proof of Theorem~2 by substituting estimates~\eqref{eq:thm_2_dx_est} and~\eqref{eq:thm_2_dp_est} into~\eqref{eq:lem_equality_subbed_simplified} and summing over $s$.
\end{proof}

\section{Gradient Descent with Back-propagation as a modification of MSA}
\label{sec:backprop_vs_msa}
Here we show that the classical gradient-descent algorithm where the gradients are computed using back-propagation~\cite{le1988theoretical} is a modification of the MSA. This was originally discussed in~\cite{li2017maximum}.
As discussed in the main paper, the reason MSA may diverge is if the arg-max step is too drastic such that the non-negative penalty terms dominate. One simple way is to make the arg-max step infinitesimal, in the appropriate direction, provided such updates provide feasible solutions. In other words, if we assume differentiability with respect to $\theta$ for all $f_t$ and that $\Theta_t$ is the whole Euclidean space, we may substitute the arg-max step with a steepest ascent step
\begin{equation}
  \theta^1_t = \theta^0_t + \eta \nabla_\theta \sum_{s=1}^{S} 
  H_t(x^{\bftheta^0}_{s,t}, p^{\bftheta^0}_{s,t+1}, \theta^0_t),
  \label{eq:grad_ascent}  
\end{equation}
for small small learning rate $\eta>0$. We show the following:
\begin{proposition}
  The MSA (Alg.~1 in the main text) with the maximization step replaced by~\eqref{eq:grad_ascent} is equivalent to gradient-descent with back-propagation on $J$. 
\end{proposition}
\begin{proof}
	As in Appendix A, WLOG we can assume $L\equiv 0$ by redefining coordinates. We have the following form for the Hamiltonian of the sample $s$
	\[
    H_t(x^{\bftheta}_{s,t}, p^{\bftheta}_{s,t+1}, \theta_t) 
    = p^{\bftheta}_{s,t+1} \cdot f(x^{\bftheta}_{s,t},\theta_t), 
	\]
	and the total loss function is $J(\bftheta) = \tfrac{1}{S}\sum_{s=1}^{S} \Phi_s(x^{\bftheta}_{s,T})$. It is easy to see that $p^{\bftheta}_{s,t} = -\tfrac{1}{S}\nabla_{x^{\bftheta}_{s,t}} \Phi_s(x^{\bftheta}_{s,T})$ (here $\nabla_{x^{\bftheta}_{s,t}}$ is the total derivative) by working backwards from $t=T$ and the fact that $\nabla_{x^{\bftheta}_{s,t}} x^{\bftheta}_{s,t+1} = \nabla_x f_t(x^{\bftheta}_{s,t},\theta_t)$. Hence, 
	\begin{align*}
  \nabla_{\theta_t} J(\bftheta) 
  =& \frac{1}{S}\sum_{s=1}^S \nabla_{x^{\bftheta}_{s,t+1}} \Phi_s(x^{\bftheta}_{s,T}) \cdot \nabla_{\theta_t} x^{\bftheta}_{s,t+1} \\
	=& \sum_{s=1}^S - p^{\bftheta}_{s,t+1} \cdot \nabla_{\theta_t} f_t(x^{\bftheta}_{s,t}, \theta_t)\\
	=& - \nabla_{\theta} \sum_{s=1}^{S} H_t(x^{\bftheta}_{s,t}, p^{\bftheta}_{s,t+1}, \theta_t)
	\end{align*}
	Hence,~\eqref{eq:grad_ascent} is simply the gradient descent step
	\[
	  \theta^{k+1}_t = \theta^k_t - \eta \nabla_{\theta_t} J(\bftheta^k).
	\]
\end{proof}
Thus, we have shown that the classical gradient descent algorithm constitute a modification of the MSA where the arg-max step is replaced by a gradient ascent step, so that (10) dominates the penalty terms in Theorem~2 in the main text (one can see this by observing that the penalty terms are now $\mathcal{O}(\eta^2)$ but the gains from steepest ascent is $\mathcal{O}(\eta)$). However, differentiability must be assumed, and moreover, $\theta^{k+1}_t$ must also be admissable, i.e.~belong to $\Theta_t$. If either condition is violated, the modification is not valid. 

\section{Implementation and Model Details}
\label{sec:model_alg_details}
A Tensorflow implementation of the binary and ternary MSA algorithm, together with code to reproduce our results are found at 
\begin{center}
  \url{https://github.com/LiQianxiao/discrete-MSA}  
\end{center}

\subsection{MSA for Binary-weight Neural Networks}
We give additional implementation details of our binary network algorithm (Alg.~2 in the main text), which is essentially Alg.~1 with the parameter update step replaced by (16). One extra step is to also keep and update an exponential moving average of $M^{\bftheta_k}_t$ and use the averaged value to update our parameters. Note that in applications, we may have some floating-point precision layers (e.g.~batch normalization layers), in which case the simplest way is to just train them using gradient descent. Also, Alg.~2 assumed that binary layers are fully-connected networks. For convolution networks, to compute $M^{\bftheta}_t$, we simply have to take gradient of $H_t$ with respect to $\theta$ (noting that $H$ is linear in $\theta$) to obtain the corresponding quantity. Before we discuss the choice of hyper-parameters in Sec.~\ref{sec:hyper-parameters}, we first give an argument for the convergence of the binary MSA algorithm in a simple setting. 

\subsubsection{Convergence of the Binary MSA for a Simple Problem}
Let us show informally that Alg.~2 in the main text converges, with an appropriate choice of regularization parameter, for a simple binary linear regression problem. The motivation here is show the importance of the added regularization terms involving $\rho_{k,t}$. 

Consider a simple linear regression problem (i.e.~linear network with $T=1$) in which the unique solution is a Binary matrix. For $s=1,\dots,S$, let $x_{s,0}\in\R^{d_0}$ be independent and have independent and identically distributed components with mean $0$ and variance $1$. These are the training samples. We shall consider the full-batch version so no exponential moving averages are applied. 

Let $\theta_0^* \in {\{-1,+1\}}^{d_0\times d_1}$ be the ground-truth, and so our regression targets are $y_s = \theta_0^* x_{s,0}$. Define the sample loss function
\[
  \Phi_s(x) := \frac{1}{2}\Vert y_s - x \Vert^2.
\]
At the $k^\text{th}$ iteration, let us denote the error vector $\delta \theta^k_0 = \theta^*_0 - \theta^k_0$. Then,
using the update rules in Alg.~2, we have
\[
  x^{\bftheta^k}_{s,1}=\theta_0^k x_{s,0} 
  \qquad
  p^{\bftheta^k}_{s,1}=\frac{1}{S}\delta\theta^k_0 x_{s,0}
\]
and so 
\[
  H_0(x^{\bftheta^k}_{s,0}, p^{\bftheta^k}_{s,1}, \theta_0) 
  = \frac{1}{S} \delta\theta^k_0x_{s,0} \cdot \theta_0 x_{s,0}
\]
The update step is then
\[
  [\theta^{k+1}_0]_{ij} = \begin{cases}
    \sgn([\delta \theta^k_0 G_S]_{ij}) 
    & \vert [\delta \theta^k_0 G_S]_{ij} \vert \geq 2S \rho_{k,0} \\
    [\theta^{k}_0]_{ij} & \text{otherwise}
  \end{cases}  
\]
where $G_S:=\tfrac{1}{S}\sum_{s=1}^S x_{s,0} x_{s,0}^T$. For large $S$, by the central limit theorem $G_S$ is approximately the identity matrix plus a small perturbation that is $\mathcal{O}(1/\sqrt{S})$ (valid for small perturbations only, large deviations will have to be bounded carefully by concentration inequalities or precise asymptotics~\cite{den2008large,boucheron2013concentration}). 
Therefore, $\delta \theta^k_0 G_S = \delta \theta^k_0 + \mathcal{O}(\Vert\delta \theta^k_0\Vert_F/\sqrt{S})$.
Taking the sign, we see that we get the correct answer (i.e.~$\delta \theta^{k+1}_1=0$) if 
$\Vert\delta \theta^k_0\Vert_F S^{-3/2} \ll \rho_{k,0} \ll \Vert\delta \theta^k_0\Vert_F S^{-1}$. Since $\Vert\delta \theta^k_0\Vert_F$ decreases as optimization proceeds, this also shows that we need to decrease $\rho_{k,t}$ as $k$ increases. 

Note that if we took the naive, unstablized MSA with $\rho_{k,0}\equiv0$ (i.e.~Alg.~1 in the main text), then it is clear that a coordinate that has the right sign ($[\delta \theta^k_0]_{ij}=0$) will continue to fluctuate because of the random signs introduced by the $\mathcal{O}(\Vert\delta \theta^k_0\Vert_F/\sqrt{S})$ term, and hence will not converge. This shows the importance of the regularization term in our algorithm. 
% Moreover, this also tells us that we should set $\rho$ to be on the same order as $M^{\bftheta^k}_0=\tfrac{1}{S}\delta \theta^k_0 G_S$. 

\subsubsection{Choice of Hyperparameters}
\label{sec:hyper-parameters}
Note that all constant factors multiplied to the hyper-parameters can be absorbed into the hyper-parameters themselves when implementing the algorithms. Hence in the following, $\rho_{k,t}$ represents the value of $2\rho_{k,t}$ in Alg.~2. 

The preceding example also shows that the regularization parameter $\rho_{k,t}$ should be suitably decreased as the optimization proceeds. We found a good heuristic is to simply set $\rho_{k,t}$ to be a constant fraction of the maximum absolute value of the components of $M^{\bftheta^k}_t$ that is not of the same sign as $\theta^k_t$. For the binary experiment, we take this constant fraction to be 0.5 for all layers.

Another hyper-parameter is the exponential moving average parameter, $\alpha_t$, which we take to be $0.999$ in all experiments. We also decay it (i.e.~making it closer to 1) as the iterations proceed. 

\subsubsection{Model Details for Experiments}
For ease of comparison, we have used almost identical set-ups as in~\cite{courbariaux2015binaryconnect}. The only difference is that we ignore the bias terms in all binary layers, resulting in slightly fewer parameters.

For the MNIST experiment, we optimize a (3xFC2048)-FC10 fully connected network. For CIFAR-10, we consider a convolutional neural network with (2xConv128)-2x2maxpool-(2xConv256)-2x2maxpool-(2xConv512)-2x2maxpool-(2xFC1024)-FC10. Lastly, for SVHN, we use the same network as CIFAR-10, but with half the number of channels in the convolution layers. All networks used ReLU activations and square-smoothed hinge loss. 
Note that the ReLU activation and the square-smoothed hinge loss are not twice differentiable, so technically it does not satisfy the assumptions in
Theorem~2. Nevertheless, we observe that the algorithm converges. Also, we tested other activations (e.g.~soft-plus, tanh) and losses (soft-max with cross entropy) and the results are similar. 
Batch-normalization is added after each affine transformation and before the non-linearity. Binary layers are trained according to Alg.~2, but batch-normalization layers have floating-point weights, and hence are trained by Adam optimizer~\cite{kingma2014adam} for simplicity. In all our experiments, no preprocessing steps are used other than scaling all input values to be between 0 and 1. We have checked that using different set-ups (e.g.~cross-entropy loss, different network structures) does not generally require retuning the parameters and the algorithm performs well. Note that however, we found that batch normalization layers are quite necessary for obtaining good performance in our algorithms, as is also the case in~\cite{courbariaux2015binaryconnect}. In the main text, Theorem 2 justifies this to a certain extent, by requiring the inputs fed to be $\mathcal{O}(1)$. 

In our comparisons with BinaryConnect~\cite{courbariaux2015binaryconnect}, we used the original code published at \url{https://github.com/MatthieuCourbariaux/BinaryConnect} with the only difference being that we changed the inference step to use binary weights (instead of full precision weights). Note that there are quite a number of regularization techniques employed here. To check their effects on the training loss, we ran the BinaryConnect code without stochastic binarization etc., but the training graphs are generally similar, hence we omit them here. 

\subsection{MSA for Ternary-weight Neural Networks}
The model setups for ternary-weight neural networks are identical as the binary ones, except we also have a parameter $\lambda_t$ for each layer that promotes sparsity. We take $\lambda_t$=1e-7 for all $t$ and all experiments. It is expected that large values will lead to sparser solutions, but with worse accuracy. We did not tune this value to find the best sparsity-performance trade-off. This is worthy of future exploration. The other hyperparameter choices are mostly identical as in the Binary case, except we take $\rho_{k,t}$ to be a smaller fraction at 0.25. 

\bibliography{ref}
\bibliographystyle{icml2018}
%%%%%%%%%%%%%%%%%%%%%%%%%%%%%%%%%%%%%%%%%%%%%%%%%%%%%%%%%%%%%%%%%%%%%%%%%%%%%%%
%%%%%%%%%%%%%%%%%%%%%%%%%%%%%%%%%%%%%%%%%%%%%%%%%%%%%%%%%%%%%%%%%%%%%%%%%%%%%%%

\end{document}